%% file: eluder_cr.tex
\pdfoutput=1
\documentclass{article}
\def\tdotoggle{0}
\input{style}
\input{macros}

\PassOptionsToPackage{square,numbers,comma}{natbib}

\usepackage[final]{neurips_2022}

\usepackage[utf8]{inputenc} 
\usepackage[T1]{fontenc}    
\usepackage{hyperref}       
\usepackage{url}            
\usepackage{booktabs}       
\usepackage{amsfonts}       
\usepackage{nicefrac}       
\usepackage{microtype}      
\usepackage{xcolor}         

\title{Understanding the Eluder Dimension}

\author{%
    Gene Li\\
    Toyota Technological\\
    Institute at Chicago\\
    \texttt{gene@ttic.edu}\\
    \And
    Pritish Kamath\\
    Google Research\\
    \texttt{pritish@alum.mit.edu}\\
    \AND
    \hspace{-3em}Dylan J. Foster\\
    \hspace{-3em}Microsoft Research\\
    \hspace{-3em}\texttt{dylanfoster@microsoft.com} \\
    \And
    \hspace{-2em}Nathan Srebro\\
    \hspace{-2em}Toyota Technological\\
    \hspace{-2em}Institute at Chicago\\
    \hspace{-2em}\texttt{nati@ttic.edu} 
}

\begin{document}

    
    
    

\maketitle

\begin{abstract}
We provide new insights on eluder dimension, a complexity measure that has been extensively used to bound the regret of algorithms for online bandits and reinforcement learning with function approximation.
First, we study the relationship between the eluder dimension for a function class and a generalized notion of {\em rank}, defined for any monotone ``activation'' $\sigma : \mathbb{R}\to \mathbb{R}$, which corresponds to the minimal dimension required to represent the class as a generalized linear model. It is known that when $\sigma$ has derivatives bounded away from $0$, $\sigma$-rank gives rise to an upper bound on eluder dimension for any function class; we show however that eluder dimension can be exponentially smaller than $\sigma$-rank. We also show that the condition on the derivative is necessary; namely, when $\sigma$ is the $\mathsf{relu}$ activation, the eluder dimension can be exponentially larger than $\sigma$-rank.
For binary-valued function classes, we obtain a characterization of the eluder dimension in terms of star number and threshold dimension, quantities which are relevant in active learning and online learning respectively.
\end{abstract}

\section{Introduction}
\citet{russo13} introduced the notion of \textbf{eluder dimension} for a function class and used it to analyze algorithms (based on the {\em Upper Confidence Bound (UCB)} and {\em Thompson Sampling} paradigms) for the multi-armed bandit problem with function approximation. Since then, eluder dimension has been extensively used to construct and analyze the regret of algorithms for contextual bandits and reinforcement learning (RL) with function approximation~\citep[see, e.g.,~][]{wen13, osband14, wang20, ayoub20, du20, foster20, dong21, feng2021provably, ishfaq2021randomized, huang2021towards}. Even though the eluder dimension has become a central technique for reinforcement learning theory, little is known about when exactly it is bounded. This paper makes progress toward filling this gap in our knowledge.

\citet{russo13} established upper bounds on eluder dimension for (i)~function classes for which inputs have finite cardinality (the ``tabular'' setting), (ii)~linear functions over $\bbR^d$ of bounded norm, and (iii)~generalized linear functions over $\bbR^d$ of bounded norm, with any activation that has derivatives bounded away from $0$.
Apart from these function classes (and those that can be embedded into these), understanding of eluder dimension has been limited.
Indeed, one might wonder whether a function class has ``small'' eluder dimension only if it can be realized as a class of (generalized) linear functions! This leads us to our first motivating question:

\begin{question}\label{q:1}
Are all function classes with small eluder dimension essentially generalized linear models?
\end{question}

Answering this question has substantial ramifications on the scope of prior work. An answer of ``yes'' would imply that the results in the aforementioned work which gives regret guarantees in terms of eluder dimension do not go beyond already-established regret guarantees for generalized linear bandit or RL settings~\citep[see, \eg][]{filippi10, li17, wang2019optimism, kveton2020randomized}. An answer of ``no'' can be construed as a positive result for RL theory, as it would indicate that existing (and future) results which use the eluder dimension apply to a richer set of function classes than generalized linear models.

To answer \Cref{q:1}, we first formally define what it means for a function class to be written as a generalized linear model (GLM). Informally, for an activation $\sigma : \bbR \to \bbR$ and a function class $\calF \subseteq (\calX \to \bbR)$, we define the \textbf{\boldmath $\sigma$-rank} to be the smallest dimension $d$ needed to express every function in $\calF$ as a generalized linear function in $\bbR^d$ with activation $\sigma$ (see \Cref{def:sigma-dc}). Intuitively, the $\sigma$-rank captures the best possible upper bound on eluder dimension that the results from \citet{russo13} can give for a given $\calF$ by treating it as a GLM with activation $\sigma$. We ask how the eluder dimension of any function class relates to its $\sigma$-rank for various activations $\sigma$. We show that the answer to \Cref{q:1} is indeed ``no'', i.e.,~there exists a function class with eluder dimension $d$ but the $\sigma$-rank \emph{for any} monotone $\sigma$ is at least $\exp(\Omega(d))$. Thus, while \citet{russo13} show that the set of function classes with small eluder dimension is a \emph{superset} of the set of GLM function classes, we (roughly speaking) show that the set of function classes with small eluder dimension is \emph{strictly larger} than the set of GLM function classes.

We also prove that the requirement from \citet{russo13} that the derivative of the activation is bounded away from $0$ is necessary in order to bound the eluder dimension of GLMs. The upper bound in the paper \cite{russo13} becomes vacuous when the activation function has zero derivative; we show a lower bound which indicates this requirement cannot be dropped. Namely, when $\sigma$ is the $\relu$ activation, we show that eluder dimension can be {\em exponentially larger} than $\sigma$-rank.

In a second line of inquiry, we study a \emph{combinatorial} version of eluder dimension. The original definition of \citet{russo13} is defined for real-valued function classes, but one can specialize the definition to binary-valued function classes, leading to a so-called \textbf{combinatorial eluder dimension}. Thus, our second motivating question is:
\begin{question}
    Can we bound the combinatorial eluder dimension, perhaps in terms of more familiar learning-theoretic quantities?
\end{question}

\noindent One might wonder: if the combinatorial eluder dimension is just a special case of the scale-sensitive version, why study it at all? Our reasons are threefold. 
\vspace{-\topsep}
\begin{itemize}[leftmargin=0.5cm]
    \item The first and most immediate reason is that we are able to show new characterizations of eluder dimension once we move to the combinatorial definition. We elucidate a fundamental connection between eluder dimension and two other well-studied learning-theoretic quantities: (1) star number, a quantity that characterizes the label complexity of \emph{pool-based active learning} \cite{hanneke15}, (2) threshold dimension, a quantity that characterizes the regret of \emph{online learning} \cite{alon2019private}.\footnote{Finiteness of the threshold dimension is equivalent to finiteness of Littlestone dimension \cite{shelah1990classification,hodges1997shorter, alon2019private}.} We believe that this new result may help us better understand how different learning tasks relate to each other.
    \item The second reason is that the combinatorial eluder dimension (or a multi-class variant of it) has already been studied for policy-based learning for contextual bandits and RL~\citep[see, \eg,][]{foster20,mou2020sample}. 
    Thus, understanding the combinatorial eluder dimension may shed light on the challenges of policy-based RL.
    \item Our last reason has more philosophical bent. Historically, the discovery of VC dimension placed statistical learning theory on solid footing. Specifically, the fundamental theorem of statistical learning allows us to precisely characterize the statistical complexity of PAC learnability in terms of the combinatorial VC dimension. The insights from understanding the role of VC dimension in classification have led researchers to develop \emph{scale-sensitive} complexity measures such as fat shattering dimension to provide sharper guarantees on learning. While we do not claim that the eluder dimension is fundamental to online RL and bandit settings, we believe that a better combinatorial understanding can lead to a better understanding of online decision making. In some sense, we are ``working backwards'' from the original scale-sensitive definition of eluder dimension to understand its combinatorial properties.
\end{itemize}

\subsection{Main contributions}
In this work, we provide several results which show when eluder dimension can be bounded (or is unbounded). Our results can be separated into two categories: (1) those pertaining to the (scale-sensitive) eluder dimension (as originally defined by \citet{russo13}); (2) those pertaining to the combinatorial eluder dimension, defined for binary-valued function classes.

First, we investigate the relationship between eluder dimension and our notion of {\em generalized rank} that captures the realizability of any function class as a GLM. In \Cref{sec:eluder}, we formally introduce the eluder dimension (as well as a related quantity of the scale-sensitive star number). In \Cref{sec:dc}, we formalize the notion of ``generalized rank''. In \Cref{sec:eluder-vs-dc}, we provide several results.

\vspace{-\topsep}
\begin{enumerate}[leftmargin=0.5cm]
\item We show that eluder dimension can be {\em exponentially smaller} than $\sigma$-rank for any monotone activation $\sigma$, not just those with derivatives bounded away from $0$ (\Cref{thm:parity-eluder-ub}).
\item We show that the condition on the derivative being bounded away from $0$ is necessary for $\sigma$-rank to be an upper bound on eluder dimension. Namely, when $\sigma$ is the $\relu$ activation, we show that eluder dimension can be {\em exponentially larger} than $\sigma$-rank (\Cref{thm:relu-eluder-lb}).\footnote{This result was independently established by \citet{dong21}.}
\end{enumerate}

\noindent In \Cref{sec:eluder-combo}, we specialize the eluder dimension to the binary-valued setting and present our results on the combinatorial eluder dimension.

\vspace{-\topsep}
\begin{enumerate}[leftmargin=0.5cm]
    \item We show that eluder dimension is finite if and only if both star number and threshold dimension are finite.
    Specifically, in \Cref{thm:equivalence} we show the following:
    \begin{align*}
        \max\{\Sdim(\calF), \mathsf{Tdim}(\calF)\} \le \Edim(\calF) \le \exp(O(\max\{\Sdim(\calF), \mathsf{Tdim}(\calF)\} )).
    \end{align*}
    Furthermore, we demonstrate that both inequalities can be tight (see \Cref{thm:tightness-ub} and discussion above it).
    \item We investigate the comparison between eluder dimension and $\sign\ddc$ and prove stronger separations than in the scale-sensitive case; namely we show examples where one quantity is finite while the other is infinite (\Cref{thm:separation}).
\end{enumerate}


\subsection{Related work}

In this section, we highlight several related works.

\paragraph{Bounds on eluder dimension.} Several papers provide bounds on eluder dimension for various function classes. The original bounds on tabular, linear, and generalized linear functions were proved by \citet{russo13} (and later generalized by \citet{osband14}). \citet{mou2020sample} provide several bounds for the combinatorial eluder dimension, mostly focusing on linear function classes. When the function class lies in an RKHS, \citet{huang2021short} show that the eluder dimension is equivalent to the notion of \emph{information gain} \cite{srinivas2009gaussian, russo2016information}, which can be seen as an infinite dimensional generalization of the fact that the eluder dimension for linear functions over $\bbR^d$ is $\tilde{\Theta}(d)$. In concurrent work, \citet{dong21} also prove an exponential lower bound on the eluder dimension for ReLU networks.

\paragraph{Applications of eluder dimension.}
The main application of the eluder dimension is to design algorithms and prove regret guarantees for contextual bandits and reinforcement learning. A few examples include the papers \cite{wen13,osband14, wang20, ayoub20, du20, foster20, jin21, feng2021provably, ishfaq2021randomized, huang2021towards, mou2020sample}. While the majority of papers prove upper bounds via eluder dimension, \citet{foster20} provided lower bounds for contextual bandits in terms of eluder dimension, if one is hoping for instance-dependent regret bounds. In addition, several works observe that eluder dimension sometimes does not characterize the sample complexity, as the guarantee via eluder dimension can be too loose \cite{huang2021going,foster2021statistical}.
Beyond the online RL setting, eluder dimension has been applied to risk sensitive RL \cite{fei2021risk}, Markov games \cite{huang2021towards, jin2021power}, representation learning \cite{xu2021representation}, and active online learning \cite{chen2021active}.

\paragraph{Other complexity measures for RL.}
We also touch upon other complexity measures which have been suggested for RL. One category is Bellman/Witness rank approach \citep[see \eg][]{jiang17, dong2020root, sun2019model}, which is generalized to bilinear classes \cite{du2021bilinear}. These complexity measures capture an interplay between the MDP dynamics and the function approximator class; in contrast, eluder dimension is purely a property of the function approximator class and can be stated (and studied) without referring to an online RL problem. \citet{jin21} define a \emph{Bellman-Eluder dimension} which captures function classes which have small Bellman rank or eluder dimension. Lastly, \citet{foster2021statistical} propose a \emph{Decision-Estimation Coefficient} and prove that it is necessary and sufficient for sample-efficient interactive learning.

\paragraph{Notions of rank.} 
The notion of rank we propose is a generalization of the classical notion of \emph{sign rank}, also called \emph{dimension complexity}. Sign rank has been studied extensively in combinatorics, learning theory, and communication complexity \citep[see \eg][and references therein]{alon85geometrical, forster02upp, arriaga06algorithmic,alon2016sign}. The norm requirements in our definition of $\sigma$-rank are related to the notion of \emph{margin complexity} \cite{arriaga06algorithmic, ben2002limitations, kamath2020approximate}.

\section{Eluder dimension and star number}\label{sec:eluder}
Eluder dimension is a ``sequential'' notion of complexity for function classes, originally defined by \citet{russo13}. 
Informally speaking, it characterizes the longest sequence of {\em adversarially chosen} points one must observe in order to accurately estimate the function value at any other point.
We consider a variant of the original definition, proposed by \citet{foster20}, that is never larger and is sufficient to analyze all the applications of eluder dimension in literature.

\begin{definition}\label{def:eluder}
For any function class $\calF\subseteq (\calX\to \bbR)$, $f^\star \in \calF$, and scale $\eps \ge 0$, the \textbf{exact eluder dimension} $\eEdim_{f^\star}(\calF, \eps)$ is the largest $m$ such that there exists $(x_1, f_1),\dots, (x_m, f_m) \in \calX\times \calF$ satisfying for all $i \in [m]$:
\begin{equation}\label{eq:eluder-def}
\inabs{f_i(x_i) - f^{\star}(x_i)} > \eps, \quad \text{and} \quad \sum_{j < i}~(f_i(x_j) - f^{\star}(x_j))^2 \le \eps^2.
\end{equation}
Then for all $\eps > 0$:

\vspace{-\topsep}
\begin{itemize}[leftmargin=0.5cm]
\item the \textbf{eluder dimension} is $\Edim_{f^\star}(\calF, \eps) = \sup_{\eps' \ge \eps} \eEdim_{f^\star}(\calF, \eps')$.%
\item $\eEdim(\calF,\eps) := \sup_{f^\star \in \calF} \eEdim_{f^\star}(\calF, \eps)$ and $\Edim(\calF,\eps) := \sup_{f^\star \in \calF} \Edim_{f^\star}(\calF, \eps)$.
\end{itemize}
\end{definition}

\noindent This definition is never larger than the original definition of \citet{russo13}, which asks for a witnessing {\em pair} of functions $f_i, f_i' \in \calF$ (the above restricts $f_i' = f^{\star}$).
Hence, all lower bounds on our variant of eluder dimension immediately apply to the original definition. Moreover, all upper bounds on eluder dimension in this paper can also be shown to hold for the other definition (unless stated otherwise).

We also consider the closely related notion of \emph{star number} defined by \citet{foster20}, which generalizes a combinatorial parameter introduced in the active learning literature by \citet{hanneke15} (we will denote it as $\Sdim$ for consistency). We study the combinatorial star number in more detail in \Cref{sec:eluder-combo}. The {\em only} difference between the definitions of eluder dimension and star number is that $\sum_{j < i}$ is replaced by $\sum_{j \ne i}$, which makes the star number a ``non-sequential'' notion of complexity.

\begin{definition}\label{def:star}
For any function class $\calF\subseteq (\calX\to \bbR)$, $f^\star \in \calF$, and scale $\eps \ge 0$, the \textbf{exact star number} $\eSdim_{f^\star}(\calF, \eps)$ is the largest $m$ such that there exists $(x_1, f_1),\dots, (x_m, f_m) \in \calX\times \calF$ satisfying for all $i\in[m]$:
\[
\inabs{f_i(x_i) - f^{\star}(x_i)} > \eps, \quad \text{and} \quad \sum_{j \ne i}~(f_i(x_j) - f^{\star}(x_j))^2 \le \eps^2.
\]
Then for all $\eps > 0$:

\vspace{-\topsep}
\begin{itemize}[leftmargin=0.5cm]
\item the \textbf{star number} is $\Sdim_{f^\star}(\calF, \eps) = \sup_{\eps' \ge \eps} \eSdim_{f^\star}(\calF, \eps')$.
\item $\eSdim(\calF,\eps) := \sup_{f^\star \in \calF} \eSdim_{f^\star}(\calF, \eps)$ and $\Sdim(\calF,\eps) := \sup_{f^\star \in \calF} \Sdim_{f^\star}(\calF, \eps)$.
\end{itemize}
\end{definition}

\noindent It immediately follows from these definitions that the star number is never larger than eluder dimension. On the other hand, the star number can be arbitrarily smaller than eluder dimension.
\begin{proposition}\label{prop:eluder-vs-star}
For all $\calF$, $f^* \in \calF$ and scale $\eps \ge 0$, it holds that\footnote{For the definition of eluder dimension considered by \cite{russo13}, an upper bound of $\min\{|\calX|, \binom{|F|}{2}\}$ holds, which can be tight. This upper bound holds because the witnessing pair of functions $(f_i, f_i')$ has to be distinct for each $i$.}
\[\eSdim_{f^*}(\calF, \eps) ~\le~ \eEdim_{f^*}(\calF, \eps) ~\le~ \min\set{|\calX|, |\calF|-1}\,.\]
\end{proposition}

\begin{proposition}[simplified from {\citet[][Prop 2.3]{foster20}}]\label{prop:eluder-vs-star2}
For the class of threshold functions given as $\Fth_n := \{f_i: [n]\to \bit \mid i\in [n+1]\}$, where $f_i(x) := \mathds{1}\inbrace{x \ge i}$, and for $f^\star = f_{n+1}$, it holds for all $\eps < 1$ that $\eSdim_{f^\star}(\calF, \eps) = 2$ and $\eEdim_{f^\star}(\calF, \eps) = n$.
\end{proposition}

\section{Generalized rank}\label{sec:dc}
{\em Dimension complexity} has been studied extensively in combinatorics, learning theory, and communication complexity \citep[see \eg][]{alon85geometrical, forster02upp, arriaga06algorithmic,alon2016sign}.
The classical notion of dimension complexity, also known as {\em sign rank}, corresponds to the smallest dimension required to embed the input space such that all hypotheses in the function class under consideration are realizable as halfspaces.
We consider a generalized notion of rank that is specified for any particular activation $\sigma : \bbR \to \bbR$, and captures to the smallest dimension required to represent the function class as a GLM when $\sigma$ is the activation. In what follows, we let $\calB_d(R) := \set{x \in \bbR^d \mid \|x\|_2 \le R}$.

\begin{definition}\label{def:sigma-dc}
For any $\sigma: \bbR \to \bbR$, the {\bf \boldmath$\sigma$-rank} of a function class $\calF\subseteq (\calX\to \bbR)$ at scale $R > 0$, denoted as $\sigma\ddc(\calF,R)$, is the smallest dimension $d$ for which there exists mappings
$\phi: \calX\to \calB_d(1)$ and $w : \calF \to \calB_d(R)$ such that\footnote{\label{ftnote:scale-interchange}Note that only the product of the scales of $\phi$ and $w$ is relevant. The definition remains equivalent if we let $\phi : \calX\to\calB_d(R_\phi)$ and $w : \calF\to\calB_d(R_w)$ for any $R_\phi$ and $R_w$ such that $R = R_\phi \cdot R_w$.}
\begin{equation}
\text{for all } (x,f) \in \calX \times \calF \ \ : \ \ 
f(x) ~=~ \sigma(\inangle{w(f), \phi(x)}),\label{eqn:sigma-rank}
\end{equation}
or $\infty$ if no such $d$ exists. For a collection of activation functions $\Sigma \subseteq (\bbR \to \bbR)$, the {\bf \boldmath $\Sigma$-rank} is
\[\Sigma\ddc(\calF, R) ~:=~ \min_{\sigma \in \Sigma} \sigma\ddc(\calF, R).\]
\end{definition}

\paragraph{Examples.} We present some examples of $\Sigma\ddc$ that motivate our definition above.
\vspace{-\topsep}
\begin{enumerate}[leftmargin=0.5cm]
\item {\bf Threshold activation.} $\sign(z)$ yields the classic notion of {\em sign-rank} (equivalent to {\em dimension complexity}, as already mentioned).
In this case, the scale $R$ is irrelevant, so we denote $\signrank(\calF) := \signrank(\calF, R)$ for any $R$. Note that this quantity is meaningful only for $\calF \subseteq (\calX \to \sbit)$.\info{While it is true that the definition remains the same for all finite $R$, the funny thing is that for any {\em finite} $R$, halfspaces over all of $\bbR^d$ do not have finite sign-rank as per our definition!}

\item {\bf Identity activation.} For $\id(z) := z$, $\id\ddc(\calF, R)$ is the smallest dimension needed to represent each $f \in \calF$ as a (norm-bounded) linear function.
We abbreviate $\dc := \id\ddc$, as this corresponds to the standard notion of rank of the matrix $(f(x))_{x, f}$ (albeit with the additional norm constraint).

\item {\bf Monotone activations.} For $L\ge \mu \ge 0$, $\calM_{\mu}^{L}$ consists of all activations $\sigma$ such that for all $z < z'$, it holds that $\mu \le \frac{\sigma(z') - \sigma(z)}{z' - z} \le L$ (for differentiable $\sigma$, this is equivalent to $\mu \le \sigma'(z) \le L$ for all $z \in \bbR$).%
\footnote{\label{ftnote:derivatives}To prove upper bounds on eluder dimension, it suffices for this condition to hold only when $|z| \le R$, \citep[see \eg][]{russo13}. Since we fix $\sigma$ in our definition first and then consider $\sigma$-rank at different scales $R$, this weaker condition complicates our definitions. Note that at any specific scale $R$, we can always modify $\sigma$ to satisfy the required constraint everywhere by extending it linearly whenever $\inabs{z} > R$. }
%
An important special case is when $\mu = 0$, and a particular $\sigma \in \calM_0^1$ of interest is the rectified linear unit (ReLU) defined as $\relu(z) := \max\set{z,0}$.
For ease of notation, we denote $\calM_{\mu} := \calM_{\mu}^1$.

While we will always be explicit about the Lipschitz constant, note that the scale of the Lipschitz constant $L$ (and $\mu$) is interchangable with the scale of $R$. In particular,
\begin{equation}\label{eq:scaling}
\calM_\mu^L\ddc(\calF, R) ~=~ \calM_{\mu/L}\ddc(\calF, RL).
\end{equation}

\item {\bf All activations.} $\SigmaA$ consists of all activations $\sigma$. We mention this notion of rank only in passing, and we will not focus on it for the rest of the paper. 
\end{enumerate}

\noindent We present a result which relates the aforementioned quantities (proof in \Cref{apdx:proof-dc-ineq}).

\begin{proposition}\label{prop:dc-ineq}
For all $\calF\subseteq (\calX\to \bbR)$, $R > 0$ and $\mu \in (0,1]$, we have:
\[
\dc(\calF, R)
~\ge~ \calM_{\mu}\ddc(\calF, R)
~\ge~ \calM_0\ddc(\calF, R)
~\ge~ \signrank(\calF) - 1,
\]
where the last inequality is meaningful only for $\calF \subseteq (\calX \to \sbit)$.
Moreover, for each inequality above, there exists a function class $\calF$ which exhibits an infinite gap between the two quantities.
\end{proposition}

\section{Eluder dimension versus generalized rank}\label{sec:eluder-vs-dc}

\newcommand{\EQ}{\cellcolor{OliveGreen!90}}
\newcommand{\GT}{\cellcolor{OliveGreen!50}}

\begin{figure}[t]
\centering
\begin{tikzpicture}
    \tikzset{every node/.style ={outer sep=.7mm}}
    \def\ylab{-2.3}
    \def\xscle{1.2}
    \def\yscle{1.2}
    \node (dc)   at (0*\xscle,1.5*\yscle)  {$\dc$};
    \node (SMdc) at (0*\xscle,0.5*\yscle)  {$\calM_\mu\ddc$};
    \node (Mdc)  at (1*\xscle,-0.5*\yscle)  {$\calM_0\ddc$};
    \node (Edim) at (-1*\xscle,-0.5*\yscle) {$\Edim$};
    \node (Sdim) at (-1*\xscle,-1.5*\yscle) {$\Sdim$};
    
    
    \path[-{Stealth[length=2mm, width=2mm]}, line width=.8pt]
    (SMdc) edge (dc)
    (Mdc)  edge (SMdc)
    (Edim) edge (SMdc)
    (Sdim) edge (Edim);
    
    \node at (0*\xscle,\ylab) {(a)};
    
    
    \node[text width=8cm, align=center] at (8,0) {
        \renewcommand{\arraystretch}{1.3}
        \begin{tabular}{!{\vrule width 1.1pt}c!{\vrule width 1.1pt}c|c|c|c|c!{\vrule width 1.1pt}}
        \noalign{\hrule height 1.1pt}
        & $\dc$ & $\calM_{\mu}\ddc$ & $\calM_0\ddc$ & $\Edim$ & $\Sdim$\\
        \noalign{\hrule height 1.1pt}
        $\dc$ & \EQ & \GT & \GT & \GT & \GT \\
        \hline
        $\calM_{\mu}\ddc$ & $\Fexp$ & \EQ & \GT & \GT & \GT \\
        \hline
        $\calM_0\ddc$ & \multicolumn{2}{c|}{$\Frelu$} & \EQ & \multicolumn{2}{c!{\vrule width 1.1pt}}{$\Frelu$} \\
        \hline
        $\Edim$ & \multicolumn{3}{c|}{\multirow{2}{*}{$\Fparity$}} & \EQ & \GT  \\
        \hhline{-~~~--} 
        $\Sdim$ & \multicolumn{3}{c|}{} & $\Fth$ & \EQ \\
        \noalign{\hrule height 1.1pt}
        \end{tabular}
    };
    
    \node at (8,\ylab) {(b)};
\end{tikzpicture}
\caption{%
	(a) Each arrow $M_1 \to M_2$ indicates that $M_1(\calF) \lesssim M_2(\calF)$ for all $\calF$, where the dependence on $R$ and $\eps$ is suppressed for clarity (see Propositions~\ref{prop:eluder-vs-star}, \ref{prop:dc-ineq}, \ref{prop:eluder-dc-sm} for precise bounds). Whenever $M_2 \to M_1$ arrow is missing, there is an example of a class $\calF$ where $M_1(\calF) \ll M_2(\calF)$.
	(b) An entry $\calF$ in $(M_1, M_2)$ means that $M_1(\calF) \ll M_2(\calF)$. \textcolor{OliveGreen}{Green} cells indicate that $M_1(\calF) \gtrsim M_2(\calF)$ for all $\calF$.
}
\label{fig:inequalities}
\end{figure}

In this section, we compare eluder dimension and star number with each notion of generalized rank: $\dc$, $\calM_{\mu}\ddc$ (for $\mu > 0$) and $\calM_0\ddc$. Our results are summarized in \autoref{fig:inequalities}.
\paragraph{\boldmath Eluder vs. $\dc$ and $\calM_{\mu}\ddc$.} \citet{russo13} and \citet{osband14} provided upper bounds on eluder dimension for linear and generalized linear function classes. For completeness, we restate this result, with a slight improvement and include the proof with precise dependence on problem parameters in \autoref{apx:eluder-proofs}. Intuitively, the generalized linear rank allows us to capture the tightest possible upper bound that the guarantees in the papers \cite{russo13, osband14} can provide on eluder dimension.

\begin{proposition}[cf. {\cite{russo13},  Prop. 6, 7; \cite{osband14}, Prop. 2, 4}]\label{prop:eluder-dc-sm}
For any function class $\calF\subseteq (\calX\to \bbR)$ and $\eps > 0$:
\vspace{-\topsep}
\begin{enumerate}[leftmargin=0.75cm, label={(\roman*)}]
    \item For all $R > 0$, $\eEdim(\calF, \eps) ~\le~ \dc(\calF, R) \cdot O\inparen{\log \frac{R}{\eps}}$.
    \item For all $L, \mu, R > 0$, $\eEdim(\calF, \eps) ~\le~ \calM_\mu^L\ddc(\calF,R) \cdot O \inparen{\frac{L^2}{\mu^2} \log \inparen{\frac{RL}{\eps}}}$.
\end{enumerate}
\end{proposition}

\noindent This result has been used to prove upper bounds on eluder dimension of various function classes beyond GLMs; for example, the class of bounded degree polynomials, by taking the feature map $\phi(x)$ to be the vector of low degree monomials.
The upper bound in Part (i) is in fact tight (up to constants) for the class of linear functions, as shown in \cref{prop:eluder-lin-tight} below. This trivially implies the optimality of the bound in Part (ii) up to the factor of $(L/\mu)^2$ which to the best of our knowledge is open.

\begin{proposition}[\cite{mahajan21}]\label{prop:eluder-lin-tight}
For any $R > 0$, $\calX := \calB_d(1)$ and $\calF := \{f_\theta : x\mapsto \inangle{\theta,x} \mid \theta \in \calB_d(R)\}$, it holds that:
\[\textstyle \eEdim(\calF, \eps) ~\ge~ \Omega\inparen{d\log\inparen{\frac{R}{\eps}}}\,.\]
\end{proposition}

\noindent For completeness, we include the proof in \Cref{apdx:proof-linear-lb}.

\paragraph{\boldmath Eluder vs. $\calM_0\ddc$.} It turns out that eluder dimension and $\calM_0\ddc$ are incomparable.
That is, there exists a function class for which eluder dimension is exponentially smaller than $\calM_0\ddc$ (and hence $\calM_\mu\ddc$ and $\dc$ by \autoref{prop:dc-ineq}).
Moreover, there exists a different function class for which eluder dimension (even the star number) is exponentially larger than $\relu\ddc$ (and hence $\calM_0\ddc$).

First, we show that the eluder dimension can be exponentially smaller than $\calM_0\ddc$ for the class of parities over $d$ bits. Thus, parities over $d$ bits exhibits an example where the eluder dimension is exponentially smaller than the best possible bound one can derive using the existing results of \Cref{prop:eluder-dc-sm}.

\begin{theorem}\label{thm:parity-eluder-ub}
For $\calX = \sbit^d$ and $\Fparity \coloneqq \set{f_S : x \mapsto \prod_{i\in S} x_i \mid S \subseteq [d]}$, it holds that
\vspace{-\topsep}
\begin{enumerate}[leftmargin=0.75cm, label={(\roman*)}]
\item $\calM_0\ddc(\Fparity, R) ~\ge~ 2^{d/2} - 1$ for all $R > 0$.
\item $\eSdim(\Fparity, \eps) ~\le~ \eEdim(\Fparity, \eps) ~\le~ d$ for all $\eps \ge 0$.
\end{enumerate}
\end{theorem}
\begin{proof}
Part (i). From \autoref{prop:dc-ineq}, we have that $\calM_0\ddc(\Fparity,R) \ge \sign\ddc(\Fparity) - 1$ for any $\sigma \in \calM_0$. The proof is now complete by noting a well known result that $\sign\ddc(\Fparity) \ge 2^{d/2}$~\citep{forster02upp}.\\[-2mm]

\noindent Part (ii). For any $x \in \sbit^d$ consider its $0$-$1$ representation $\widetilde{x} \in \bbF_2^d$ (representing $+1$ by $0$ and $-1$ by $1$). All functions in $\Fparity$ can be simply viewed as linear functions over $\bbF_2$. Namely, any parity function is indexed by a vector $a \in \bbF_2^d$, with $f_a(x) := (-1)^{\inangle{a, \widetilde{x}}}$.

Note that $\eEdim(\Fparity, \eps) = 0$ for all $\eps \ge 2$. For any $\eps < 2$, suppose $\eEdim_{f^\star}(\Fparity, \eps) = m$, witnessed by $(x_1, f_{a_1}), \ldots, (x_m, f_{a_m}) \in \sbit^d$ and $f^\star = f_{a^{\star}}$. We have

\vspace{-\topsep}
\begin{itemize}[leftmargin=0.5cm]
\item $f_{a_{i}}(x_i) \ne f_{a^{\star}}(x_i)$, and 
\item $f_{a_{i}}(x_j) = f_{a^{\star}}(x_j)$ for all $j < i$ : since $\sum_{j < i} (f_{a_{i}}(x_j) - f_{a^{\star}}(x_j))^2 < \eps^2 < 4$ iff all terms are $0$.
\end{itemize}
Thus, we have $\inangle{a_{i}-a^{\star}, \widetilde{x}} = 0$ for all $\widetilde{x} \in \bbF_2\text{-}\mathrm{span}(\set{\widetilde{x}_1, \ldots, \widetilde{x}_{i-1}})$.
But $\inangle{a_{i}-a^{\star}, \widetilde{x}_i} = 1$ and hence $\widetilde{x}_i$ is linearly independent of $\set{\widetilde{x}_1, \ldots, \widetilde{x}_{i-1}}$ over $\bbF_2^d$.
Thus, $\set{\widetilde{x}_1, \ldots, \widetilde{x}_m}$ are all linearly independent over $\bbF_2^d$, and hence $m \le d$.
\end{proof}

\noindent The bound in part (ii) of \Cref{thm:parity-eluder-ub} was also calculated by \citet[][Prop.~3]{mou2020sample}.

\medskip
\noindent Next, we show a separation in the other direction for eluder dimension vs.~$\calM_0\ddc$ using the ReLU function class. Thus, we cannot hope to remove the requirement for the activation function $\sigma$ to be strictly monotonically increasing in \Cref{prop:eluder-dc-sm} (ii) for bounding the eluder dimension.

\begin{theorem}\label{thm:relu-eluder-lb}
	Let $R > 0$ and $\calX = \calB_d(1)$. Define
	\begin{align*}
	    \Frelu := \set{f_{\theta,b} : x \mapsto \relu(\inangle{\theta,x} + b) \mid \|\theta\|^2 + b^2 \le R^2}.
	\end{align*}
	It holds that
	
\vspace{-\topsep}
\begin{enumerate}[leftmargin=0.75cm, label={(\roman*)}]
		\item $\calM_0\ddc(\Frelu, R) \le \relu\ddc(\Frelu, R) \le d+1$,
		\item $\eEdim(\Frelu, \eps) \ge \eSdim(\Frelu, \eps) \ge \inparen{\frac{R}{4\eps}}^{d/2}$ for all $\eps \in (0, \frac{R}{4})$.
	\end{enumerate}
\end{theorem}
\begin{proof}
Part (i) is immediate from the definition. We show Part (ii) in the special case of $R=2$; the general case follows by relatively scaling $\eps$, since $\relu$ is {\em homogeneous}, namely, $\relu(ax) = a \cdot \relu(x)$. Consider any $U \subseteq \calX$ such that $\|u\| = 1$ and $\inangle{u, v} \le 1-\eps$ for all $u, v \in U$. It holds that $\eSdim_{f^\star}(\Frelu, \eps) \ge |U|$ when $f^\star$ is the identically zero function, since the function $f_u(x) = \relu(\inangle{u,x} - (1-\eps))$ is such that $f_u(v) = 0$ for all $v \in U \smallsetminus \set{u}$, whereas $f_u(u) = \eps$. A standard sphere packing argument shows that such a set $U$ exists with $\inabs{U} \ge (1/2\eps)^{d/2}$ for all $\eps < 1/2$. In particular, the $\delta$-packing number of the unit sphere is at least $(1/\delta)^d$ \cite[Cor.~4.2.13]{vershynin2018high}. Thus, we can find $(1/\delta)^d$ points such that each pair $u,v$ satisfies $\norm{u-v} \ge \delta$, or equivalently $\inangle{u,v} \le 1- \delta^2/2$. Setting $\delta = \sqrt{2\eps}$ proves the claimed lower bound.
\end{proof}
\noindent \autoref{thm:relu-eluder-lb} was independently shown by \citet[][Thm.~5.1]{dong21}.

\medskip

We remark that while we considered the variant of eluder dimension as defined by \citet{foster20}, the lower bound on eluder dimension in \autoref{thm:relu-eluder-lb} immediately holds for the notion defined by \citet{russo13}. On the other hand, the upper bound on eluder dimension in \autoref{thm:parity-eluder-ub} can be shown to hold even with the definition of \citet{russo13} (by replacing every instance of $a^{\star}$ by $a_i'$).

\section{Relationships between combinatorial measures}\label{sec:eluder-combo}
In this section, we specialize the eluder dimension to binary-valued function classes and prove several additional characterizations that relate this combinatorial eluder dimension to other learning-theoretic quantities. First, we (re)define these learning-theoretic quantities. The first two definitions (of combinatorial eluder dimension and star number respectively) are the specialization of the scale-sensitive versions (cf.~\Cref{def:eluder} and \ref{def:star}) to the binary-valued function setting. We abuse notation by dropping the argument $\eps$ from previous definitions in order to be consistent.

\begin{definition}\label{def:combinatorial}
Fix any function class $\calF \subseteq (\calX \to \{1,-1\})$ and any $f^\star \in \calF$.

\vspace{-\topsep}
\begin{itemize}[leftmargin=0.5cm]
\item The \textbf{combinatorial eluder dimension} w.r.t.~$f^\star$, denoted $\Edim_{f^\star}(\calF)$, is defined as the largest $m$ such that there exists $(x_1, f_1),\dots, (x_m, f_m) \in \calX \times \calF$ satisfying for all $i\in[m]$:
\begin{align*}
    f_i(x_i) \ne f^\star(x_i), \quad \text{and}\quad \text{for all} \ j<i: \quad f_i(x_j) = f^\star(x_j).
\end{align*}
\item The \textbf{star number} w.r.t.~$f^\star$, denoted $\Sdim_{f^\star}(\calF)$, is defined as the largest $m$ such that there exists $(x_1, f_1),\dots,  (x_m, f_m) \in \calX \times \calF$ satisfying for all $i\in[m]$:
\begin{align*}
    f_i(x_i) \ne f^\star(x_i), \quad \text{and}\quad \text{for all} \ j\ne i: \quad f_i(x_j) = f^\star(x_j).
\end{align*}
\item The \textbf{threshold dimension} w.r.t.~$f^\star$, denoted $\Tdim_{f^\star}(\calF)$, is defined as the largest $m$ such that there exists $(x_1, f_1),\dots, (x_m, f_m) \in \calX \times \calF$ satisfying for all $i\in[m]$:
\begin{align*}
    \text{for all} \ k \ge i: \quad f_i(x_k) \ne f^\star(x_k), \quad \text{and} \quad \text{for all} \ j<i: \quad f_i(x_j) = f^\star(x_j). 
\end{align*}
\end{itemize}
As before, we define the combinatorial eluder dimension (resp.~star number and threshold dimension) to be $\Edim(\calF) \coloneqq \sup_{f\in \calF} \Edim_f(\calF)$ ($\Sdim(\calF)$ and $\Tdim(\calF)$ are defined similarly).
\end{definition}

\begin{figure}[t]
\begin{center}
\input{figs/witness-seq}
\caption{Illustration of witnessing sequences of length $6$ for eluder dimension, star number and threshold dimension with respect to $f^\star = 0$ (we use the $0$/$1$ representation of functions for clarity). `*' in the eluder witness sequence refers to a free value, either 0 or 1.}
\label{fig:eluder-star-thresh}
\end{center}
\end{figure}

\noindent Let us pause to unpack these definitions and give some background. In fact, the star number definition stated above is the original definition of \citet{hanneke15}, who give tight upper and lower bounds on the label complexity of \emph{pool-based active learning} via the star number $\Sdim(\calF)$ and show that almost every previously proposed complexity measure for active learning takes a worst case value equal to the star number. Roughly speaking, the star number corresponds to the number of ``singletons'' one can embed in a function class; that is, the maximum number of functions that differ from a base function $f^
\star$ at exactly one point among a subset of the domain $\{x_1,\dots, x_m\} \subseteq \calX$. 

The threshold dimension has recently gained attention due to its role in proving an equivalence relationship between private PAC learning and online learning \citep[see, e.g.,~][]{alon2019private, bun2020equivalence}. We slightly generalize the definition of \citet{alon2019private} to allow for any base function $f^\star$, in the spirit of the other two definitions. A classical result in model theory provides a link between the threshold dimension and Littlestone dimension (which we denote $\mathsf{Ldim}$), a quantity which is both necessary and sufficient for online learnability \cite{ben2009agnostic, alon2021adversarial}. In particular, results by \citet{shelah1990classification} and \citet{hodges1997shorter} show that for any binary-valued $\calF$ and any $f^\star \in \calF$:
\begin{align*}
    \floor{\log \Tdim_{f^\star}(\calF)} \le \mathsf{Ldim} \le 2^{\Tdim_{f^\star}(\calF)}.
\end{align*}
A combinatorial proof of this fact can be found in Thm.~3 of \citet{alon2019private}.\footnote{\citet{alon2019private} prove the result for $f^\star(x) = -1$, but it is easy to extend their proof to hold for any $f^\star$.} Thus, finiteness of threshold dimension is necessary and sufficient for online learnability (albeit in a much weaker, ``qualitative'' sense).

\subsection{A qualitative equivalence}
Now, we prove a rather surprising characterization that closely ties all three quantities in \Cref{def:combinatorial} together. Our result implies that for any binary-valued function class, finiteness of the combinatorial eluder dimension is \emph{equivalent} to finiteness of both star number and threshold dimension.
\begin{theorem}\label{thm:equivalence}
For any function class $\calF \subseteq (\calX \to \{1,-1\})$ and any $f^\star \in \calF$, the following holds:
\begin{align*}
    \max\{\Sdim_{f^\star}(\calF), \Tdim_{f^\star}(\calF) \} \le \Edim_{f^\star}(\calF) \le 4^{\max\{ \Sdim_{f^\star}(\calF), \Tdim_{f^\star}(\calF) \} }.
\end{align*}
\end{theorem}

\noindent The proof of \Cref{thm:equivalence} can be found in \Cref{sec:proof-equivalence}. The lower bound is trivial by examining \Cref{def:combinatorial}. To prove the upper bound, we rely on a novel connection to Ramsey theory. In particular, we show that sequences $(x_1, f_1), \dots (x_m, f_m)$ which witness $\Edim_{f^\star} = m$ form a bijection with red-blue colorings of the complete graph $K_m$, while subsequences of the witnessing eluder sequence that are valid star number witnesses or threshold dimension witnesses can be interpreted as monochromatic colorings of subgraphs of $K_m$ (see \cref{fig:eluder-star-thresh}). Thus, applying classical bounds from Ramsey theory implies the result.

\Cref{thm:equivalence} has an exponential gap between the upper and lower bounds. Can we improve either of the inequalities? The lower bound cannot be improved, by considering the simple examples over $\calX = [n]$ of the singleton class $\Fsing_n \coloneqq \{x \mapsto \mathds{1}\{x = i\} \mid i \in [n+1]\}$ and the threshold class $\Fth_n \coloneqq \{x \mapsto \mathds{1}\{x \ge i\} \mid i \in [n+1]\}$. We also show that the upper bound cannot be improved, for example, to $\Edim(\calF) \le \mathrm{poly}(\Sdim(\calF), \Tdim(\calF))$ in general.

\begin{theorem}\label{thm:tightness-ub}
For every $N > 0$, there exists a function class $\calF_N$ such that $\Edim(\calF_N) = N$ and $\max\{\Sdim(\calF), \Tdim(\calF)\} < c \cdot \log_2 N$, where $c> 1/2$ is some absolute numerical constant.
\end{theorem}

\noindent The proof of \Cref{thm:tightness-ub} can be found in \Cref{apdx:pf-tightness-ub}. It relies on the probabilistic method to show the existence of a randomly constructed $\calF$ which satisfies the desired properties.

\subsection{\boldmath Comparisons with $\signrank$}

In this section, we investigate the comparison of these combinatorial quantities (eluder, star, threshold) with $\signrank$, and examine whether we can prove stronger separations. One direction is clear; for the class of linear classifiers in $\bbR^d$, the combinatorial eluder dimension, star number, and threshold dimension are all infinite. However, we ask if the other separation is also possible: can we construct $\calF$ where eluder/star/threshold are finite but $\signrank = \infty?$ We have already provided an explicit exponential separation: \Cref{thm:parity-eluder-ub} shows that for the function class $\Fparity$, we had $\Edim(\Fparity) = \Sdim(\Fparity) = d$ but $\signrank(\Fparity) \ge 2^{d/2}$. (One can also show that $\Tdim(\Fparity) = d$.) We are able to show stronger (but nonconstructive) separations by extending the probabilistic techniques of \citet{alon2016sign}, who recently provided similar separations for VC dimension versus $\signrank$. In view of \Cref{prop:dc-ineq}, this result also provides a separation for the scale-sensitive \Cref{def:eluder}. 

\begin{theorem}\label{thm:separation} For every $N > 0$, there exists a function class $\calF_N \subseteq ([N] \to \{1,-1\})$ such that $\Edim_1(\calF_N) = 4$ and $\signrank(\calF_N) \ge \Omega(N^{1/9}/\log N)$, where $1$ is shorthand for the all 1s function.
\end{theorem}

\noindent The proof of \Cref{thm:separation} can be found in \Cref{apdx:proof-separation}. It is straightforward to replace the reference function $f^\star(x) = 1$ with any fixed reference function $f^\star: [N]\to \{1,-1\}$. First, we use Lemma 22 of \citet{alon2016sign} which bounds the number of distinct matrices with $\signrank=r$; then using a probabilistic argument we show that there must be many (more) matrices with $\Edim_1=4$, so at least one of them must have large $\signrank$.

The careful reader might notice that we do not prove the existence of a function class where $\Edim(\calF)$ is constant and the $\signrank$ is infinite; instead we prove the weaker statement that a function class exists with $\Edim$ w.r.t.~any \emph{fixed} function $f^\star$ is bounded. We conjecture that the stronger statement holds; see \Cref{apdx:proof-separation-discussion} for more details.

\subsection{Back to scale-sensitive?}
It is natural to ask if our results can be extended back to the scale-sensitive definitions. First, we require a scale-sensitive version of threshold dimension. One proposal is the following:
\begin{definition}\label{def:threshold}
For any function class $\calF\subseteq (\calX\to \bbR)$, $f^\star \in \calF$, and scale $\eps \ge 0$, the \textbf{exact threshold dimension} $\eTdim_{f^\star}(\calF, \eps)$ is the largest $m$ such that there exists $(x_1, f_1),\dots, (x_m, f_m) \in \calX\times \calF$ satisfying for all $i\in[m]$:
\begin{align*}\label{eq:eluder-def}
\forall j \ge i \ : \inabs{f_i(x_j) - f^{\star}(x_j)} > \eps, \quad \text{and} \quad \sum_{j < i}~(f_i(x_j) - f^{\star}(x_j))^2 \le \eps^2.
\end{align*}
Then for all $\eps > 0$:

\vspace{-\topsep}
\begin{itemize}[leftmargin=0.5cm]
\item the \textbf{threshold dimension} is $\Tdim_{f^\star}(\calF, \eps) = \sup_{\eps' \ge \eps} \eTdim_{f^\star}(\calF, \eps')$.%

\item $\eTdim(\calF,\eps) := \sup_{f^\star \in \calF} \eTdim_{f^\star}(\calF, \eps)$ and $\Tdim(\calF,\eps) := \sup_{f^\star \in \calF} \Tdim_{f^\star}(\calF, \eps)$.
\end{itemize}
\end{definition}

\noindent \Cref{def:threshold} mirrors the scale-sensitive definitions for eluder and star; it also recovers the combinatorial definition when $\calF$ is binary-valued. By definition, the relationship that $\Edim_{f^\star}(\calF, \eps) \ge \max\{\Sdim_{f^\star}(\calF, \eps), \Tdim_{f^\star}(\calF, \eps) \}$ for every $\calF, f^\star \in \calF, \eps >0$ is trivial. However, one cannot hope to prove the corresponding upper bound under this definition. For example, for any $\eps > 0$, take the function class which is represented by the $N\times N$ matrix:
\begin{align*}
    f_j(x_i) = \begin{cases} 0 &i < j\\
    \eps &i=j\\
    0.99 \eps & i > j.
    \end{cases}
\end{align*}
It is easy to see that $\Edim_0(\calF, \eps) = N$, while $\Sdim_0(\calF, \eps) = 2$ and $\Tdim_0(\calF,\eps) = 1$. Notice that this class is still ``threshold-like'', but \Cref{def:threshold} does not capture this for said value of $\eps$. Generally speaking, it is unclear how to carry over the intuition from Ramsey theory that applies in the combinatorial case to the scale-sensitive case; we leave this to future work.



\gene{Discuss: A flaw with the scale-sensitive measures is that the relationships with the generalizations of VC dimension and Littlestone dimension break down, due to a certain ``centering'' issue.}

\begin{ack}
We thank Gaurav Mahajan for allowing us to include the proof of \autoref{prop:eluder-lin-tight} \citep{mahajan21}.
We thank Akshay Krishnamurthy, Tengyu Ma, and Ruosong Wang for helpful discussions.
GL was partially supported by NSF award IIS-1764032.
PK was partially supported by NSF BIGDATA award 1546500.
Part of this work was done while GL, PK, and DF were participating in the Simons program on the Theoretical Foundations of Reinforcement Learning.
\end{ack}

\newpage
\DeclareUrlCommand{\Doi}{\urlstyle{sf}}
\renewcommand{\path}[1]{\small\Doi{#1}}
\renewcommand{\url}[1]{\href{#1}{\footnotesize\Doi{#1}}}
\bibliographystyle{abbrvnat}
\bibliography{eluder_cr}

\newpage

\newpage
\appendix

\section{Proof of \Cref{prop:dc-ineq}}\label{apdx:proof-dc-ineq}

The proof uses the following result, which is straightforward by definition.

\begin{proposition}\label{prop:dc-prop} $\Sigma$-rank satisfies the following for all $\calF\subseteq (\calX \to \bbR)$.
\vspace{-\topsep}
\begin{enumerate}[leftmargin=0.75cm, label={(\roman*)}]
    \item For all $R < R'$ : $\Sigma\ddc(\calF,R) \ge \Sigma\ddc(\calF,R')$.
    \item For all $\Sigma_1 \subseteq \Sigma_2$ and $R > 0$: $\Sigma_1\ddc(\calF, R) \ge \Sigma_2\ddc(\calF, R)$.
\end{enumerate}
\end{proposition}

\paragraph{Proof of \Cref{prop:dc-ineq}.}
The first two inequalities follow from immediately from \autoref{prop:dc-prop}. For the last inequality, let $\phi : \calX \to \bbR^d$ and $w : \calF \to \bbR^d$ be the mappings that witness $\sigma\ddc(\calF,R) = d$ for some $\sigma \in \calM$. Thus, we have that for all $(x,f) \in \calX \times \calF$, it holds that, $f(x) = \sigma(\inangle{w(f), \phi(x)})$. Let $t \in \bbR$ be any value such that $\sigma(t) = 0$. From monotonicity of $\sigma$ and the fact that $\calF$ is $\sbit$-valued, we have that for all $(x,f) \in \calX \times \calF$, $f(x) = \sign(\inangle{w(f), \phi(x)} - t)$.
Thus, $\sign\ddc(\calF) \le \sigma\ddc(\calF, R) + 1$. We next move to the examples witnessing the separations.
\vspace{-\topsep}
\begin{itemize}[leftmargin=0.5cm]
\item {\boldmath $\dc \gg \SigmaSM_{\mu}\ddc$}: Consider $\calX = [0,1]$ and $\Fexp := \set{f_\theta : x\mapsto \sigma(\theta \cdot x) \mid |\theta| \le 1}$, where $\sigma(\cdot)$ is defined piecewise as follows: $\sigma(z) = e^z$ for $z \in [0,1]$; outside of $[0,1]$, we extend the function linearly with slope $1$ when $z< 0$ and slope $e$ when $z> 1$. Since $\sigma \in \calM_1^e$, we have $\calM_1^e\ddc(\Fexp,1) = 1$ and hence $\calM_{1/e}\ddc(\calF,e) = 1$ (from \autoref{eq:scaling}).

Consider the points $\{x_j := j/d \mid j \in \set{1,\ldots,d}\}$ and the functions $\{f_{\theta_i} := f_{i/d} \mid j\in \set{1,\ldots,d}\}$. The matrix $A$ given by $A_{ij} := f_{\theta_i}(x_j) = (e^{j/d^2})^{i}$ is a Vandermonde matrix, and hence $\mathrm{rank}(A) = d$.
Since $d$ can be chosen to be arbitrarily large, it follows that $\dc(\Fexp, R) = \infty$ for all $R > 0$.

\item {\boldmath $\SigmaSM_{\mu}\ddc \gg \calM_0\ddc$}: Consider $\calX = [0,1]$ and $\Frelu := \set{f_{a,b} : x \mapsto \relu(a x + b) \mid a^2 + b^2 \le 1}$, the class of ReLUs with biases in $1$ dimension. Clearly, $\calM_0\ddc(\Frelu, \sqrt{2}) = 2$.

Suppose for contradiction that for some $\sigma \in \calM_{\mu}$, it holds that $\sigma\ddc(\Frelu,R) = d$ with mappings $\phi: \calX \to \calB_d(1)$ and $w : \Frelu \to \calB_d(R)$ for some $R > 0$. Consider $n = d+2$ points $0 < x_1 < x_2 < \cdots < x_n < 1$. For each $i$, $f_i(x) := \relu(x-x_{i-1}) \in \Frelu$ (let $x_0 := 0$) satisfies $f_i(x_j) = \sigma(\inangle{w(f_i), \phi(x_j)}) = 0$ for all $j < i$, and $f_i(x_j) = \sigma(\inangle{w(f_i), \phi(x_j)}) > 0$ for all $j \ge i$. That is, $\inangle{w(f_i), \phi(x_j)} = \sigma^{-1}(0)$ for all $j < i$ (since $\sigma:\bbR\to\bbR$ is strictly monotone, $\sigma^{-1}(0)$ is uniquely defined) and $\inangle{w(f_i), \phi(x_j)} > \sigma^{-1}(0)$ for all $j \ge i$.
Consider the matrix $A \in \bbR^{n \times n}$ given by $A_{ij} := \inangle{w(f_i), \phi(x_j)}$. By definition, $\mathrm{rank}(A) \le d$. On the other hand, we have that $A - \sigma^{-1}(0) \cdot J$ is an upper-triangular matrix  with non-zero diagonals, where $J \in \bbR^{n \times n}$ is the all-$1$s matrix. It follows that $\mathrm{rank}(A) \ge n-1$ and hence $n \le d+1$. This is a contradiction and hence $\calM_\mu(\Frelu, R) = \infty$ for all $\mu, R > 0$. 

\item {\boldmath $\calM_0\ddc \gg \mathrm{sign}\ddc$}: Consider $\calX = [0,1]$ and $\Fth := \{f_t : x \mapsto \sign(x-t) \mid t \in [0,1]\}$, the class of Thresholds. Clearly, $\signrank(\Fth) = 2$.

We briefly sketch the argument showing $\calM_0\ddc(\calF,R) = \infty$ for any $R > 0$. Suppose for some $\sigma \in \calM_0$ it holds that $\sigma\ddc(\Fth,R) = d$. Then it is possible to realize $\Fth$ as halfspaces {\em with margin}, since $\sigma$ is $1$-Lipschitz. This implies that $\Fth$ is online learnable with a finite mistake bound (via the Perceptron algorithm). This is a contradiction since $\Fth$ is not online learnable with a finite mistake bound \citep[see \eg][Lemma~21.6 and Ex.~21.4]{shalev2014understanding}.\qedhere
\end{itemize}

\section{Proof of \texorpdfstring{\autoref{prop:eluder-dc-sm}}{Proposition~\ref{prop:eluder-dc-sm}}}\label{apx:eluder-proofs}

\autoref{prop:eluder-dc-sm} follows from \autoref{clm:eluder-lin-ub} and \autoref{clm:eluder-genlin-ub} below, where for clarity, we keep track of the norms of $\phi$ and $w$ separately. Our improvement comes about from the following lemma, which is inspired by a similar step in \cite{russo13}.

\begin{lemma}[Inspired by \cite{russo13}]\label{lem:ineq}
	For all $k\ge 1$ and $\alpha, \beta > 0$, if $(1+\alpha)^k \le 1+\beta k$, then $k \le \frac{e}{e-1} \cdot \frac{1+\alpha}{\alpha} \cdot \ln \inparen{\frac{2\beta(1+\alpha)}{\alpha}}$.
\end{lemma}
\begin{proof}
We consider two cases. In both cases, we use that $\ln(1 + \alpha) \ge \frac{\alpha}{1+\alpha}$ holds for all $\alpha > 0$. Also note that $(1+\alpha)^k > 1+k\alpha$ and hence $\beta > \alpha$.\\

\noindent {\em Case 1:} If $\beta k < 1$, we have $(1+\alpha)^k \le 2$ and hence $k \le \frac{\ln 2}{\ln(1+\alpha)} \le \ln 2 \cdot \frac{1+\alpha}{\alpha}$. Since $\beta > \alpha$, we have that $\ln 2 \le \frac{e}{e-1} \cdot \ln \inparen{\frac{2\beta(1+\alpha)}{\alpha}}$, thereby completing the proof for this case.\\

\noindent {\em Case 2:} If $\beta k \ge 1$, we have $(1+\alpha)^k \le 2 \beta k$. Taking logarithms, we have $k \ln (1+\alpha) ~\le~ \ln k + \ln 2\beta$. Hence, we have $\frac{k\alpha}{1+\alpha} ~\le~ \ln \inparen{\frac{k\alpha}{1+\alpha}} + \ln \inparen{\frac{2 \beta (1+\alpha)}{\alpha}}$. Using $\ln x \le \frac{x}{e}$ for all $x \ge 0$, we get for $x = \frac{k \alpha}{1+\alpha}$ that $k \cdot \frac{\alpha}{1+\alpha} \cdot \inparen{1 - \frac{1}{e}} ~\le~ \ln\inparen{\frac{2\beta(1+\alpha)}{\alpha}}$, thereby completing the proof.
\end{proof}

\begin{claim}\label{clm:eluder-lin-ub}
Suppose $\dc(\calF,R_\phi R_w) = d$ is witnessed by mappings $\phi:\calX\to \calB_d(R_\phi)$ and $w: \calF\to \calB_d(R_w)$. Then $\eEdim(\calF, \eps) ~\le~ \frac{3e}{e-1} \cdot d \cdot \log \inparen{\frac{24 R_\phi^2 R_w^2}{\eps^2}}$ for all $\eps < R_\phi R_w$.
\end{claim}
\begin{proof}
Suppose $\eEdim_{f^\star}(\calF,\eps) = m$ witnessed by the sequence $(x_1, f_1), \dots, (x_m, f_m) \in \calX\times \calF$, for some $f^\star \in \calF$ and $\eps > 0$. Denote $w_i := w(f_i) -w(f^\star)$, and $\phi_i := \phi (x_i)$. It follows that $w_i \in \calB_d(2R_w)$ and $\phi_i \in \calB_d(R_\phi)$. From \autoref{eq:eluder-def}, we have that for all $i\in [m]$:
\begin{equation}\label{eq:cvx_edim}
	\max_{w\in \bbR^d} \inbrace{\inabs{\inangle{w, \phi_i}}: \sum_{j< i} \inangle{w, \phi_i}^2 \le \eps^2, \|w\|_2 \le 2R_w} ~>~ \eps
\end{equation}

\noindent Let $V_i := \lambda I + \sum_{j<i} \phi_j \phi_j^\top$. The above equation implies that for all $i\in[m]$:
\[ \max_{w\in \bbR^d} \inbrace{\inabs{\inangle{w, \phi_i}}: \norm{w}_{V_i}^2 \le \eps^2 + \lambda\cdot 4 R_w^2} ~>~ \eps,\]
which via convex duality and setting $\lambda := \eps^2/(4R_w^2)$, further implies that for all $i\in[m]$:
\[\norm{\phi_i}_{V_i^{-1}}^2 > \frac{\eps^2}{\eps^2 + \lambda \cdot 4 R_w^2} = \frac{1}{2}, \]

\noindent We will use a potential argument to track the quantity $\det(V_i)$. First, we have the upper bound:
\[\det(V_{m+1}) \le \inparen{\frac{\tr(V_{m+1})}{d}}^d \le \inparen{\frac{\lambda d+ mR_\phi^2}{d}}^d = \lambda^d \inparen{1+ \frac{mR_\phi^2}{\lambda d}}^d = \lambda^d \inparen{1+ \frac{m}{d}\cdot \frac{4R_w^2 R_\phi^2}{\eps^2}}^d,\]
using, AM-GM inequality, the linearity of trace, and the definition of $\lambda$.
\noindent We also have the lower bound, using the Matrix Determinant Lemma:
\[\det(V_{m+1}) = \det(V_{m} + \phi_m \phi_m^\top) = \det(V_m)\cdot (1+ \norm{\phi_{m}}_{V_m^{-1}}^2) \ge \lambda^d \inparen{3/2}^{m},\]
where the last step follows by induction. Combining the upper and lower bounds, we get that $(3/2)^{m/d} \le 1+\frac{m}{d}\cdot  \frac{4R_w^2R_\phi^2}{\eps^2}$. The claim follows from an application of \autoref{lem:ineq} with $k= \frac{m}{d}$, $\alpha = \frac12$ and $\beta = \frac{4R_\phi^2 R_w^2}{\eps^2}$.
\end{proof}

\begin{claim}\label{clm:eluder-genlin-ub}
For $\sigma \in \calM_{\mu}^L$, suppose $\sigma\ddc(\calF,R_\phi R_w) = d$ is witnessed by mappings $\phi:\calX\to \calB_d(R_\phi)$ and $w: \calF\to \calB_d(R_w)$. Then $\eEdim(\calF, \eps) ~\le~ \frac{3e}{e-1} \cdot d \cdot \frac{L^2}{\mu^2} \cdot \log \inparen{\frac{24 R_\phi^2 R_w^2 L^2}{\eps^2}}$ for all $\eps < R_\phi R_w L$.
\end{claim}
\begin{proof}
Suppose $\eEdim_{f^\star}(\calF,\eps) = m$ witnessed by the sequence $(x_1, f_1), \dots, (x_m, f_m) \in \calX\times \calF$, for some $f^\star \in \calF$ and $\eps > 0$.
Denote $w_i := w(f_i) -w(f^\star)$, and $\phi_i := \phi (x_i)$. It follows that $w_i \in \calB_d(2R_w)$ and $\phi_i \in \calB_d(R_\phi)$. Since $\sigma \in \calM_{\mu}^L$, we have for any $w_1, w_2, x\in \bbR^d$:
\[\mu \inabs{\inangle{w_1 - w_2, x}}
~\le~ \inabs{\sigma(\inangle{w_1, x}) - \sigma(\inangle{w_2,x})}
~\le~ L\inabs{\inangle{w_1 - w_2, x}}.\]
Therefore, Eq. (\ref{eq:cvx_edim}) can be replaced with:
\begin{equation}
	\max_{w\in \bbR^d} \inbrace{\inabs{\inangle{w, \phi_i}}: \sum_{j< i} \inangle{w, \phi_i}^2 \le \eps^2/\mu^2, w\in \calB_d(2R_w)} > \eps/L.
\end{equation}
Following the same steps with $\lambda := \eps^2/(4R_w^2\mu^2)$ and $V_i$ defined as before, we can further show that:
\[ \max_{w\in \bbR^d} \inbrace{\inabs{\inangle{w, \phi_i}}: \norm{w}_{V_i}^2 \le \eps^2/\mu^2 + \lambda\cdot 4 R_w^2} > \eps/L,\]
which implies the bound:
\[\norm{\phi_i}_{V_i^{-1}}^2 > \frac{\eps^2/L^2}{\eps^2/\mu^2 + \lambda \cdot 4 R_w^2} = \frac{1}{2}\cdot \frac{\mu^2}{L^2}.\]
Using similar upper and lower bounds on $\det(V_m)$ gives us that $(1+\mu^2/(2L^2))^{m/d} \le 1+\frac{m}{d}\cdot  \frac{4R_w^2R_\phi^2\mu^2}{\eps^2}$; the proof again concludes with an application of \autoref{lem:ineq} with $k= \frac{m}{d}$, $\alpha = \frac{\mu^2}{2L^2}$ and $\beta = \frac{4R_\phi^2 R_w^2\mu^2}{\eps^2}$.
\end{proof}

\paragraph{Discussion of prior work.} We clarify the differences between \autoref{prop:eluder-dc-sm} and the corresponding propositions in \citet{russo13, osband14}.%

Proposition 6 in \cite{russo13} considers the setting exactly as in \autoref{clm:eluder-lin-ub}. In our notation, the stated bound has the form $\dc(\calF,R_\phi R_w) \cdot O(\log \inparen{3 + \frac{12R_w^2}{\eps^2}})$; the factor of $R_\phi$ is missing inside the log term. As explained in \autoref{def:sigma-dc} (\autoref{ftnote:scale-interchange}), only the product $R_\phi R_w$ is relevant as the scale of $\phi$ and $w$ is interchangeable.

Proposition 7 in \cite{russo13} considers the setting exactly as in \autoref{clm:eluder-genlin-ub}. In our notation, the stated bound there has the form $\calM_{\mu}^L\ddc(\calF,R_\phi R_w) \cdot O(\frac{L^2}{\mu^2} \cdot \log \inparen{\frac{3L^2}{\mu^2} + \frac{L^2}{\mu^2} \cdot \frac{12 R_w^2 L^2}{\eps^2}})$; again the factor of $R_\phi$ is missing. Also, the factor of $L$ in $R_w^2 L^2/\eps^2$ is improvable to $\mu$. 

Proposition 4 in \cite{osband14} considers the setting analogous to \autoref{clm:eluder-genlin-ub}, but for vector-valued function classes, that is, $\calF \subseteq (\calX \to \bbR^k)$. In the special case of $k=1$, their bound in our notation, has the form $\calM_{\mu}^L\ddc(\calF,R_\phi R_w) \cdot O(\frac{L^2}{\mu^2} \cdot \log \inparen{\frac{L^2}{\mu^2} + \frac{L^2}{\mu^2} \cdot \frac{R_\phi^2 R_w^2}{\eps^2}})$. The term $R_\phi R_w$ should be $R_\phi R_w \mu$. As shown in \autoref{eq:scaling}, it is possible to make $R_\phi R_w$ arbitrarily small while keeping $L/\mu$ fixed. 

Lastly, we note that  \autoref{lem:ineq} is slightly different than the corresponding inequality used in \citet{russo13,osband14} (which has $(1+\beta)$ in place of $2\beta$). This allows us to remove the additive terms of $3$ and $3\frac{L^2}{\mu^2}$ inside the log factor for the $\dc$ case and the $\calM_\mu^L\ddc$ case respectively. In the $\calM_\mu^L\ddc$ case, this gives a nontrivial improvement in some regime of parameters; namely the bounds of \cite{russo13,osband14} would only give us a term of $\log\inparen{\frac{L}{\mu} + \frac{R_\phi R_w L}{\eps}}$, which can be loose when $\mu$ is very small.


\section{Proof of \Cref{prop:eluder-lin-tight} }\label{apdx:proof-linear-lb}

We exhibit a sequence $(x_1, \theta_1), \dots (x_m, \theta_m)$ that witnesses the claimed lower bound on eluder dimension with $\theta^\star = 0$. It suffices to consider the case of $R=1$, as this is just a matter of scaling relative to $\eps$.
First, consider the case of $d=1$. For any $\alpha \in (\eps, \sqrt{2}\eps)$ and $k := \floor{ \log_2(1/\alpha)}$, let $x_i := 1/2^{(k-i)}$ and $\theta_i = \alpha \cdot 2^{k-i}$ for $i \in \set{0, \ldots, k}$.
For each $i$, it holds that $\theta_i x_i = \alpha > \eps$ and $\sum_{j<i} (\theta_i x_j)^2 \le \alpha^2/2 < \eps^2$. Since $|x_i| \le 1$ and $|\theta_i| \le 1$ we get $\eEdim(\calF, \eps) \ge k+1 \ge \Omega(\log(1/\eps))$. 

For $d > 1$, consider $d$ copies of the above $1$ dimensional sequence repeated in each dimension. Namely, consider the sequence $(x_{ij}, \theta_{ij})_{i \in [d], j \in [k]}$ with $x_{ij} := {\bm e}_i / 2^{k-j}$ and $\theta_{ij} := \alpha2^{k-j} \cdot {\bm e}_i$ (where ${\bm e}_i$ is the $i$-th standard basis vector). Since $x_{ij}, \theta_{ij} \in \calB_d(1)$, we have $\eEdim(\calF,\eps) \ge d(k+1) = \Omega\inparen{d \log (1/\eps)}$.

\section{Proof of \Cref{thm:equivalence}}\label{sec:proof-equivalence}

\paragraph{Proof of lower bound.} The lower bound is straightforward from the definition, since any sequence of $(x_1,f_1), \dots, (x_m, f_m)$ that witness $\Sdim_{f^\star}(\calF) = m$ or $\Tdim_{f^\star}(\calF) = m$ must also be valid ``eluder sequences''; so the eluder dimension can only be larger.

\paragraph{Proof of upper bound.}
Fix $\calF$ and $f^\star \in \calF$. Let $(x_1, f_1), \dots (x_m, f_m) \in \calX \times \calF$ be the sequence which witnesses $\Edim_{f^\star}(\calF) = m$. To prove the bound, we will show that there exists a subset of size at least $k \ge \log_4 m$ which witnesses $\Sdim_{f^\star} = k$ or $\Tdim_{f^\star} = k$.

This follows from a connection to Ramsey theory. A visualization of the proof is depicted in \Cref{fig:proof-illustration}. Recall that diagonal Ramsey number $R(k,k)$ is defined as the smallest $m$ such that every red-blue labeling of the edges of the graph $K_m$ contains a monochromatic subgraph $K_k$. 

In addition, define $E(k,k)$ as the smallest $m$ such that any eluder sequence $(x_1, f_1), \dots, (x_m, f_m)$ contains a subsequence $(x_{i_1}, f_{i_1}), \dots (x_{i_k}, f_{i_k})$ which witnesses $\Sdim_{f^\star}(\calF) \ge k$ or $\Tdim_{f^\star}(\calF) \ge k$. We claim that $E(k,k) \le R(k,k)$.

To see this, note that there exists a bijection between colorings of $K_m$ and eluder sequences. Every eluder sequence $(x_1, f_1), \dots (x_m, f_m)$ can be used to construct a red-blue coloring of $K_m$ as follows. For every edge $e_{ij}$ for $i > j \in [m]$, we color it red if $f_j(x_i) = f^\star(x_i)$ and blue otherwise. Observe that if there exists a subsequence $(x_{i_1}, f_{i_1}), \dots (x_{i_k}, f_{i_k})$ which witnesses $\Sdim_{f^\star} \ge k$, then the subgraph comprised of the vertices $i_1, \dots, i_k$ in the coloring of $K_m$ must be monochromatic red. Likewise, if a subsequence witnesses $\Tdim_{f^\star} \ge k$, then the subgraph must be monochromatic blue. Thus, if $m$ is such that if every coloring of $K_m$ induces a monochromatic coloring $K_k$, then for any eluder sequence, we can always find a subsequence that witnesses $\Sdim_{f^\star} \ge k$ or $\Tdim_{f^\star} \ge k$. This shows that $E(k,k) \le R(k,k)$.

The proof concludes by applying the classical bound $R(k,k) \le 4^{k}$ \cite[see, e.g.,~][]{mubayi2017survey}.

\begin{figure}[t]
\centering
\subfigure[Eluder matrix]{
\includegraphics[scale=0.45, trim={0cm 27cm 55cm 0cm},clip]{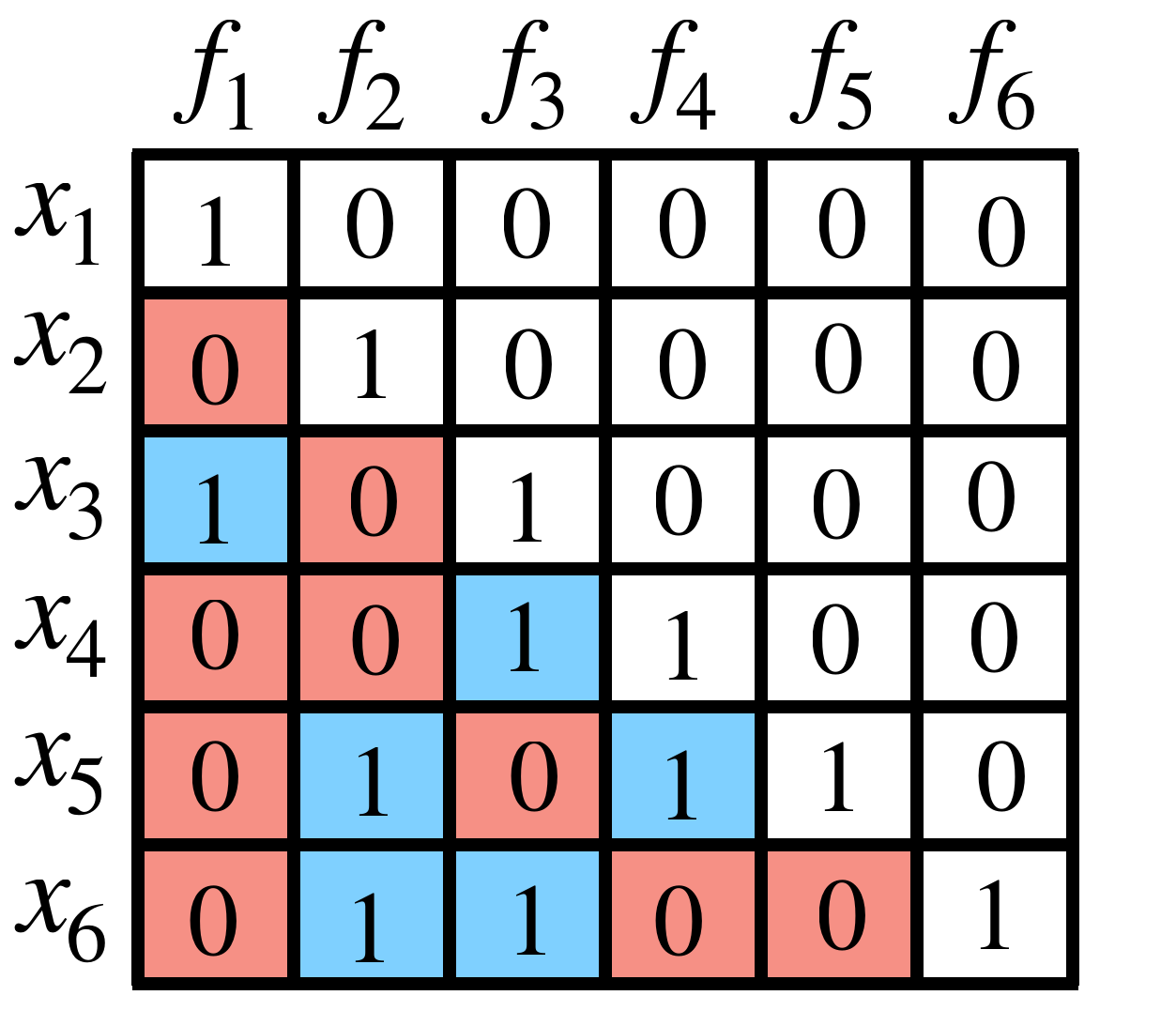}
}
\hspace{1em}
\subfigure[Ramsey graph]{
\includegraphics[scale=0.45, trim={0cm 28cm 55cm 0cm},clip]{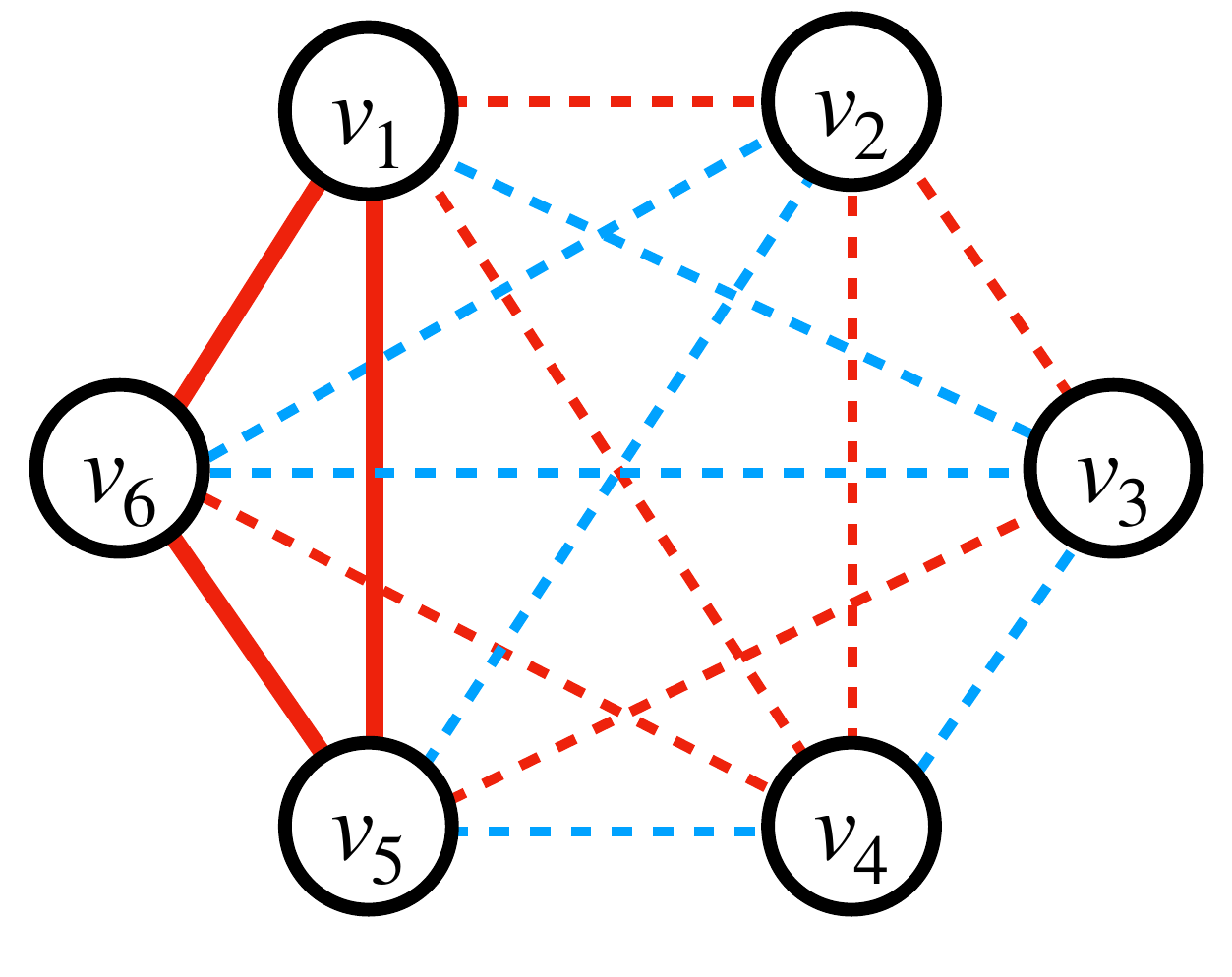}
}
\caption{An example illustrating the connection between the upper bound and Ramsey numbers. Left: a sequence $\{(x_1, f_1), \dots (x_6, f_6)\}$ witnessing $\Edim_0(\calF) = 6$, represented in matrix form. (We switch to 0/1-valued function classes for clarity.) Right: In the corresponding graph $K_6$, we color the graph edges $e_{ij}$ to be \textcolor{red}{red} if $f_j(x_i) = 0$ and \textcolor{blue}{blue} if $f_j(x_i) = 1$. Since $R(3,3) = 6$, we are guaranteed a subgraph $K_3$ which is monochromatic; in this example, the subgraph is given by the vertices $\{v_1, v_5, v_6\}$. Red subgraphs indicate sequences which witness $\Sdim_0(\calF)$; blue subgraphs witness $\Tdim_0(\calF)$. In this case, $\{(x_1, f_1), (x_5, f_5), (x_6, f_6)\}$ witnesses $\Sdim_0(\calF) \ge 3$.}
\label{fig:proof-illustration}
\end{figure}

\section{Proof of \Cref{thm:tightness-ub}}\label{apdx:pf-tightness-ub}
We will construct $\calF_N \subseteq ([N]\to \{1,-1\})$ randomly, such that $\inabs{\calF_N} = N+1$ and $\Edim_{1}(\calF_N) = N$. Note that it is equivalent to define an $N\times (N+1)$ sign matrix $B$, representing the values of $\calX \times \calF$ with entry $B_{ij} = f_j(x_i)$. Let $B$ be randomly drawn according to the following distribution:
\begin{align*}
    B_{ij} \sim \begin{cases}
    1 & i < j,\\
    -1 & i = j,\\
    \mathrm{Rad}(1/2) & i > j.
    \end{cases}
\end{align*}
By construction, with probability 1, $(x_1, f_1) \dots (x_N, f_N)$ is a valid sequence witnessing $\Edim(\calF_N) = \Edim_1(\calF_N) = N$.

We now have to argue that there exists some matrix $C \in \{1,-1\}^{N\times (N+1)}$ such that the equivalent function class $\calF_N$ has small threshold dimension and star number. We use the following two lemmas to simplify the requirement that $\calF_N$ have small threshold dimension and star number with respect to all base functions $f^\star\in \calF$ to just considering the base functions $f^\star(x) = -1$ and $f^\star(x) = 1$.

\begin{lemma}\label{lem:dim-ub}
For any $\calF$ and $\dim\in \{\Edim, \Tdim, \Sdim\}$, we have  $\dim(\calF) \le \dim_{1}(\calF) + \dim_{-1}(\calF)$. 
\end{lemma}

\begin{lemma}\label{lem:thres-ub}
For any $\calF$, $\Tdim(\calF) \le 2 \Tdim_{1}(\calF)$.
\end{lemma}

We set up some additional notation. Denote ${\bf I} = (i_1, i_2, \dots, i_k)$ and ${\bf J} = (j_1, j_2, \dots, j_k)$ to be $k$-length sequences of distinct elements from $[N]$. For any two sequences ${\bf I}$, ${\bf J}$, we use $(x_{\bf I}, f_{\bf J}) \coloneqq ((x_{i_1}, f_{j_1}), \dots (x_{i_k}, f_{j_k}))$. We define valid star sequences to be any $(x_{\bf I}, f_{\bf J})$ which witness $\Sdim_1(\calF_N) = k$ or $\Sdim_{-1}(\calF_N) = k$, and we define valid threshold sequences to be any $(x_{\bf I}, f_{\bf J})$ which witnesses $\Tdim_1(\calF_N) = k$.

Define the random variable $X_k$ to be the number of valid star or threshold sequences, i.e.,~$X_k \coloneqq \inabs{ \inbrace{ ({\bf I}, {\bf J}) \colon (x_{\bf I}, f_{\bf J}) \text{ is valid star or threshold sequence} } }.$ By linearity of expectation, we have
\begin{align*}
    \Ex[X_k] \le N^{k}(N+1)^{k} \cdot \max_{ {\bf I}, {\bf J} \subset [N] } p_{\bf{I}, \bf{J}}, \quad \text{where } p_{\bf{I}, \bf{J}} \coloneqq \mathbb{P}\left[ (x_{\bf I}, f_{\bf J}) \text{ is a valid star or threshold sequence} \right].
\end{align*}
Now we apply the following lemma to upper bound the expectation.

\begin{lemma}\label{lem:prob-upper-bound}
For all $\bf{I}$, $\bf{J}$, we have $ p_{\bf{I}, \bf{J}} \le 3 \cdot 2^{-k(k-1)/2}$.
\end{lemma}

\noindent We apply \Cref{lem:prob-upper-bound} to the previous display. When $k = \Omega(\log_2 N)$, we have $\Ex[X_k] < 1$. By the probabilistic method, there must exist an $N\times (N+1)$ valued matrix $C$ such that the corresponding $\calF_N \subseteq ([N]\to \{1, -1\}$ has $\Sdim_1(\calF_N) < O(\log_2 N)$, $\Sdim_{-1}(\calF_N) < O(\log_2 N)$, and $\Tdim_1(\calF_N) < O(\log_2 N)$, but $\Edim_{1}(\calF_N) = N$. By \Cref{lem:dim-ub} and \ref{lem:thres-ub}, this concludes the proof.\qed

\paragraph{Proof of \Cref{lem:dim-ub}.}
We prove the result for star number; the result for eluder dimension and threshold dimension can be shown with a similar argument. Fix any $f^\star \in \calF$, and let $(x_1, f_1), \dots (x_m, f_m)$ denote the sequence which witnesses $\Sdim_{f^\star}(\calF) = m$. Let $I_{+} \subseteq [m]$ denote the indices $i$ for which $f^\star(x_i) = 1$ and $I_{-} \subseteq [m]$ denote the indices $i$ for which $f^\star(x_i) = -1$. By definition of star number, $f_j(x_j) = -1$ for every $j\in I_{+}$ and $f_k(x_j) = f^\star(x_j) = 1$ for every $j\ne k \in I_{+}$. Thus we know that $\{(x_k, f_k): k\in I_{+}\}$ is a valid sequence which witnesses $\Sdim_{1}(\calF) \ge \inabs{I_{+}}$. Similarly $\{(x_k, f_k): k\in I_{-}\}$ is a valid sequence which witnesses $\Sdim_{-1}(\calF) \ge \inabs{I_{-}}$. Thus we have shown that $\Sdim_{f^\star}(\calF) = m \le \Sdim_{1}(\calF) + \Sdim_{-1}(\calF)$. Taking the supremum on the LHS yields the claim.\qed

\paragraph{Proof of \Cref{lem:thres-ub}.} Fix any $f^\star \in \calF$, and let $(x_1, f_1), \dots (x_m, f_m)$ denote the sequence which witnesses $\Tdim_{f^\star}(\calF) = m$. Again let $I_{+} \subseteq [m]$ denote the indices $i$ for which $f^\star(x_i) = 1$ and $I_{-} \subseteq [m]$ denote the indices $i$ for which $f^\star(x_i) = -1$. Either $\inabs{I_{+}} \ge m/2$ or $\inabs{I_{-}} \ge m/2$. We break into cases.

\textit{Case 1.} If $\inabs{I_{+}} \ge m/2$, then taking the subsequence indexed by $I_{+}$ already shows that $\Tdim_1(\calF) \ge m/2 = \Tdim_{f^\star}(\calF)/2$, and we are done.

\textit{Case 2.} If $\inabs{I_{-}} \ge m/2$, then let us consider the subsequence indexed by $I_{-}$. We reindex it to call it $(x_1, f_1), \dots (x_k, f_k)$, where $k = \inabs{I_{-}}$. Observe that the sequence
\begin{align*}
(x_k, f^\star), (x_{k-1}, f_k), (x_{k-2}, f_{k-1}), \dots, (x_1, f_2)
\end{align*}
witnesses $\Tdim_1(\calF) \ge k \ge m/2 = \Tdim_{f^\star}(\calF)/2$.

Thus in both cases we have shown that $\Tdim_1(\calF) \ge \Tdim_{f^\star}(\calF)/2$; taking the supremum yields the claim.\qed

\paragraph{Proof of \Cref{lem:prob-upper-bound}.}
Fix any $\bf{I}$, $\bf{J}$ to be $k$-length subsequences of $[N]$. In order for $(x_{\bf I}, f_{\bf J})$ to be a valid star sequence w.r.t.~$f^\star(x) = 1$, the following properties of the matrix $B$ must hold:
\vspace{-\topsep}
\begin{enumerate}[leftmargin=0.5cm]
    \item For every $r \in [k]$, $B_{i_r, j_r} = -1$.
    \item For every $r, s \in [k]$ such that $r\ne s$, $B_{i_r, j_s} = 1$.
\end{enumerate}

\noindent In order for $(x_{\bf I}, f_{\bf J})$ to be a valid star sequence w.r.t.~$f^\star(x) = -1$, we just flip the values in the above two properties.

Likewise, in order for $(x_{\bf I}, f_{\bf J})$ to be a valid threshold sequence w.r.t.~$f^\star(x)=1$, the following properties of the matrix $B$ must hold:
\vspace{-\topsep}
\begin{enumerate}[leftmargin=0.5cm]
    \item For every $r \in [k]$, $B_{i_r, j_r} = -1$.
    \item For every $r, s \in [k]$ such that $r < s$, $B_{i_r, j_s} = 1$.
    \item For every $r, s \in [k]$ such that $r \ge s$, $B_{i_r, j_s} = -1$.
\end{enumerate}

First, we will prove that the probability that $(x_{\bf I}, f_{\bf J})$ is a valid star sequence w.r.t.~$f^\star(x)=1$, as well as the probability that $(x_{\bf I}, f_{\bf J})$ is a valid threshold sequence w.r.t.~$f^\star(x)=1$ are both $2^{-k(k-1)/2}$. For any $r \in [k]$, if $i_r <  j_r$, then by construction of $B$ we know that $f_{j_r}(x_{i_r}) = 1$, so $(x_{\bf I}, f_{\bf J})$ cannot be a valid star sequence or threshold sequence w.r.t~$f^\star(x)=1$. Henceforth, assume $i_r \ge j_r$ for all $r\in[k]$. Now define indices $r_1, r_2, \dots, r_k$ as the permutation of $[k]$ such that $i_{r_1}> i_{r_2}> \dots > i_{r_k}$. For any $p < q$, we have $i_{r_p} > i_{r_q} \ge j_{r_q}$. Thus for every $p < q$, the corresponding entry $(i_{r_p}, j_{r_q})$ is sampled from $\mathrm{Rad}(1/2)$. In order for $(x_{\bf I}, f_{\bf J})$ to be a valid star sequence, all of these must take the value of 1; likewise in order for  $(x_{\bf I}, f_{\bf J})$ to be a valid threshold sequence, all of these must take the value of $-1$. Since there are $k(k-1)/2$ of these, we have the desired result.

Now we bound the probability that $(x_{\bf I}, f_{\bf J})$ is a valid star sequence w.r.t.~$f^\star(x)=-1$. Note that because the definition of star number is permutation-invariant, we can assume that $i_1 > i_2 > \dots > i_k$ without loss of generality. Consider the pair $(i_1, j_1)$. We require $B_{i_1, j_1} = 1$, so either $i_1 > j_1$ or $i_1 < j_1$. Since we require $B_{i_2, j_1} = -1$, we cannot have $i_1 < j_1$, so we must have $i_1 > j_1$. Using a similar argument, we must have $i_r > j_r$ for all $r\in[k-1]$. Thus, there must be at least $k(k-1)/2$ random entries in the submatrix given by $({\bf I}, {\bf J})$, all of which must take value $-1$ in order for $(x_{\bf I}, f_{\bf J})$ is a valid star sequence w.r.t.~$f^\star(x)=-1$.

By union bound we get $p_{\bf{I}, \bf{J}} \le 3 \cdot 2^{-k(k-1)/2}$, thus proving the result.\qed

\section{Proof of \Cref{thm:separation}}\label{apdx:proof-separation}

For simplicity, let us consider only function classes of size $N$. It is equivalent to reason about $N\times N$ sign matrices which define the values that $\calX \times \calF$ take. We slightly abuse notation to define the $\signrank$ of an $N\times N$ matrix $S$ to be
\begin{align*}
    \signrank(S) \coloneqq \inbrace{\mathrm{rank}(M) : M\in \bbR^{N\times N}, \ \sign(M_{ij}) = S_{ij} \text{ for all } i,j\in[N]}.
\end{align*}
We also define $\Edim_1(S)$ similarly: $\Edim_1(S)$ is the maximum $k$ such that we can find two $\bf{I}$, $\bf{J}$ which are $k$-length subsequences of $[N]$ such that the matrix $S$ restricted to $\bf{I}$, $\bf{J}$ has $-1$ on the diagonal and $+1$ above the diagonal.

The proof relies on the following key lemma, which bounds the number of matrices with $\signrank$ at most $r$.

\begin{lemma}[e.g., Lemma 22 of \cite{alon2016sign}]\label{lem:num-sign-matrices}
Let $r \le N/2$. The number of $N\times N$ sign matrices with sign rank at most $r$ does not exceed $2^{O(rN\log N)}$.
\end{lemma}

\noindent In order to prove the result, we use a probabilistic argument to show that there must exist many distinct $N\times N$ matrices with $\Edim_1 = 4$.

\begin{lemma}\label{lem:eluder-matrices-count}
The number of $N\times N$ sign matrices with $\Edim_1 \le 4$ is at least $2^{\Omega(N^{10/9})}$. 
\end{lemma}

\noindent The above lemmas imply that there must exist at least one sign matrix with $\Edim_1(S) \le 4$ and $\signrank(S) \ge \Omega(N^{1/9}/\log N)$. This proves \Cref{thm:separation}, assuming \Cref{lem:eluder-matrices-count} which we now prove.

\paragraph{Proof of \Cref{lem:eluder-matrices-count}.}
Define the set of $E_5$-light matrices as the set of $5\times 5$ sign matrices which are always $+1$ above the diagonal. We claim there exists an $N\times N$ sign matrix $C$ which (1) contains no $E_5$-light matrices and (2) has at least $\Omega(N^{10/9})$ entries that are $+1$. Such a matrix $C$ has $\Edim_1(C) \le 4$; moreover changing any $+1$ to $-1$ in $C$ will not increase $\Edim_1$.

Let $B$ be a random $N\times N$ sign matrix with each entry $+1$ with probability $p \coloneqq 1/(2N^{8/9})$. Define the random variable
\begin{align*}
    X \coloneqq \text{(\# $+1$'s in $C$)} - \text{(\# $E_5$-light matrices in $C$)}.
\end{align*}
Then $\Ex[X] \ge N^2 p - N^{10} p^{10} \ge \Omega(N^{10/9})$. Take some matrix with value of $X$ at least the expectation and change a $+1$ to a $-1$ in every $E_5$-light matrix to get $C$. Since this does not affect the value of $X$, we know that the resulting matrix $B$ has $\Omega(N^{10/9})$ entries that are $+1$.

Since changing any $+1$ to $-1$ in $C$ will not increase $\Edim_1$, we see that there are at least $2^{\Omega(N^{10/9})}$ distinct sign matrices with $\Edim \le 4$.\qed

\subsection{A stronger separation?}\label{apdx:proof-separation-discussion}
Our result does not fully show the separation between eluder dimension and $\signrank$. In the random construction used in \Cref{lem:eluder-matrices-count}, it could be the case that the matrix $C$ we pick satisfies $\Edim_1(C) \le 4$, but there could be some other $f^\star$ (column of $C$) such that $\Edim_{f^\star}(C) = \omega(1)$. 

We conjecture that the stronger separation result also holds:

\begin{conjecture}
There exists absolute constants $k \in \bbN$ and $c > 0$ such that the following hold. For every $N > 0$, there exists a function class $\calF_N \subseteq ([N] \to \{1,-1\})$ such that $\Edim(\calF_N) \le k$ and $\signrank(\calF_N) \ge \Omega(N^{c}/\log N)$.
\end{conjecture}

In light of \Cref{thm:equivalence}, it suffices to show that there exists some function class $\calF_N$ such that $\max\{\Sdim(\calF_N), \Tdim(\calF_N) \} \le k$ and $\signrank(\calF_N) \ge \Omega(N^{c}/\log N)$.

Consider the easier problem of showing the separation for threshold dimension. Here there is no difficulty. The key step is to apply \Cref{lem:thres-ub} to reduce the problem to showing the result with respect to a \emph{single} $f^\star$. It follows as a corollary of \Cref{thm:separation} that there exists a function class $\calF_N$ such that $\Tdim(\calF_N) \le 8$ and $\signrank(\calF_N) \ge \Omega(N^{1/9}/\log N)$. (Using a more direct analysis, it is possible to improve the constants 8 and $1/9$.)

Showing the separation for star number (and eluder dimension) is a different story. We do not have a direct analogue of \Cref{lem:thres-ub} for star number and eluder dimension. The weaker \Cref{lem:dim-ub} allows us to reduce two showing the separation for two functions; but it is unclear how to leverage this reduction to extend the construction in the proof of  \Cref{lem:eluder-matrices-count}.

\gene{Open problem: (For star number separation) Do there exist at least $2^{\Omega(N^{1+\eps})}$ sign matrices that do not contain $I_k$ or $-I_k$ as submatrices? (For eluder separation) Do there exist at least $2^{\Omega(N^{1+\eps})}$ sign matrices that do not contain $I_k$ or $-I_k$ or $T_k$ as submatrices?}

\end{document}

%% file: style.tex
\usepackage{amsfonts,amsthm,amssymb,mathtools,thmtools,thm-restate,bm}

\usepackage[usenames,dvipsnames,table]{xcolor}
\usepackage{enumitem}
\usepackage{xspace}
\usepackage{array}
\usepackage{multirow,multicol}
\usepackage[USenglish]{babel}
\usepackage{csquotes}
\usepackage{dsfont}
\usepackage{hhline}
\usepackage{nicefrac}
\usepackage{subfigure}

\usepackage{tikz}
\usetikzlibrary{arrows.meta}

\usepackage{hyperref}
\hypersetup{
	colorlinks=true,
	linkcolor=Blue,
	citecolor=black!30!ToCgreen,
	filecolor=Maroon,
	urlcolor=Maroon
}

\usepackage[capitalize,noabbrev,nameinlink]{cleveref}
\addto\extrasenglish{}
\addto\extrasenglish{}
\addto\extrasenglish{}
\newcommand{\secref}[1]{\hyperref[#1]{\S\ref{#1}}}

\usepackage[colorinlistoftodos,prependcaption,textsize=scriptsize]{todonotes} 
\providecommand{\tdotoggle}{1}
\ifnum\tdotoggle=1
\paperwidth=\dimexpr \paperwidth + 3cm\relax
\oddsidemargin=\dimexpr\oddsidemargin + 1.6cm\relax
\evensidemargin=\dimexpr\evensidemargin + 1.6cm\relax
\marginparwidth=\dimexpr \marginparwidth + 1.6cm\relax
\fi
\newcommand{\mytodo}[1]{\ifnum\tdotoggle=1{#1}\fi}

\newcommand{\info}[1]{\mytodo{\todo[linecolor=OliveGreen,backgroundcolor=OliveGreen!25,bordercolor=OliveGreen]{#1}}}

\newcommand{\gene}[1]{\mytodo{\todo[linecolor=myGold,backgroundcolor=myGold!25,bordercolor=myGold]{GL: #1}}}

\newcommand{\tableoftodos}{\ifnum\tdotoggle=1 \listoftodos[Comments/To Do's] \fi}

\definecolor{Gred}{RGB}{219, 50, 54}
\definecolor{Ggreen}{RGB}{60, 186, 84}
\definecolor{Gblue}{RGB}{72, 133, 237}
\definecolor{Gyellow}{RGB}{247, 178, 16}
\definecolor{ToCgreen}{RGB}{0, 128, 0}
\definecolor{myGold}{RGB}{231,141,20}
\definecolor{myBlue}{rgb}{0.19,0.41,.65}
\definecolor{myPurple}{RGB}{175,0,124}

\DeclareMathOperator{\Ex}{\mathbb{E}}

\DeclareMathOperator{\sign}{sign}

\def\eps{\varepsilon}

\declaretheorem[name=Theorem]{theorem}

\declaretheorem[name=Lemma,sibling=theorem]{lemma}

\declaretheorem[name=Claim,sibling=theorem]{claim}

\declaretheorem[name=Proposition,sibling=theorem]{proposition}

\declaretheorem[name=Conjecture,sibling=theorem]{conjecture}

\declaretheorem[name=Definition]{definition}
\declaretheorem[name=Question]{question}

\newcommand{\set}[1]{\left \{ #1 \right \}}
\newcommand{\bit}{\{0,1\}}
\newcommand{\sbit}{\{-1,1\}}
\newcommand{\inabs}[1]{\left | #1 \right |}
\newcommand{\inparen}[1]{\left ( #1 \right )}

\newcommand{\inbrace}[1]{\left \{ #1 \right \}}
\newcommand{\inangle}[1]{\left \langle #1 \right \rangle}

\newcommand{\norm}[1]{\lVert #1 \rVert}

\DeclarePairedDelimiter\floor{\lfloor}{\rfloor}

\newcommand{\bbF}{\mathbb{F}}

\newcommand{\bbN}{\mathbb{N}}

\newcommand{\bbR}{\mathbb{R}}

\newcommand{\calB}{\mathcal{B}}

\newcommand{\calF}{\mathcal{F}}

\newcommand{\calM}{\mathcal{M}}

\newcommand{\calX}{\mathcal{X}}


%% file: macros.tex
\newcommand{\eEdim}{\underline{\mathsf{Edim}}}
\newcommand{\Edim}{\mathsf{Edim}}

    \newcommand{\eg}{e.g.\xspace}

\newcommand{\eSdim}{\underline{\mathsf{Sdim}}}
\newcommand{\Sdim}{\mathsf{Sdim}}

\newcommand{\eTdim}{\underline{\mathsf{Tdim}}}
\newcommand{\Tdim}{\mathsf{Tdim}}

\renewcommand{\dim}{\mathsf{dim}}

\newcommand{\relu}{\mathsf{relu}}
\renewcommand{\sign}{\mathsf{sign}}
\newcommand{\dc}{\mathsf{rk}}

\newcommand{\signrank}{\sign\text{-}\dc}
\newcommand{\ddc}{\text{-}\dc}
\newcommand{\id}{\mathsf{id}}
\newcommand{\SigmaA}{\Sigma^{\mathrm{all}}}

\newcommand{\SigmaSM}{\calM}

\newcommand{\Fth}{\calF^{\mathrm{th}}}
\newcommand{\Fsing}{\calF^{\mathrm{sing}}}
\newcommand{\Frelu}{\calF^{\mathrm{relu}}}
\newcommand{\Fexp}{\calF^{\mathrm{exp}}}
\newcommand{\Fparity}{\calF^{\oplus}}

\newcommand{\tr}{\mathrm{tr}}

%% file: figs/witness-seq.tex
\newcommand{\drawgrid}{%
	\foreach \i in {0,...,6} {
		\draw[black] (\mx+\i*\xgap-0.5*\xgap, \my+0.5*\ygap) -- (\mx+\i*\xgap-0.5*\xgap, \my-5.5*\ygap);
		\draw[black] (\mx-0.5*\xgap,\my-\i*\ygap+0.5*\ygap) -- (\mx+5.5*\xgap, \my-\i*\ygap+0.5*\ygap);
	}
	\foreach \i in {1,...,6} {
		\node at (\mx-\xgap, \my-\i*\ygap+\ygap) {\scriptsize$x_{\i}$};
		\node at (\mx+\i*\xgap-\xgap, \my+\ygap) {\scriptsize$f_{\i}$};
	}
}

\begin{tikzpicture}
	\def \xgap{0.45}
	\def \ygap{0.45}
	
	\def \mx{0}
	\def \my{0}
	\drawgrid
	\node at (\mx+0*\xgap, \my-0*\ygap) {1};
	\node at (\mx+0*\xgap, \my-1*\ygap) {*};
	\node at (\mx+0*\xgap, \my-2*\ygap) {*};
	\node at (\mx+0*\xgap, \my-3*\ygap) {*};
	\node at (\mx+0*\xgap, \my-4*\ygap) {*};
	\node at (\mx+0*\xgap, \my-5*\ygap) {*};
	\node at (\mx+1*\xgap, \my-0*\ygap) {0};
	\node at (\mx+1*\xgap, \my-1*\ygap) {1};
	\node at (\mx+1*\xgap, \my-2*\ygap) {*};
	\node at (\mx+1*\xgap, \my-3*\ygap) {*};
	\node at (\mx+1*\xgap, \my-4*\ygap) {*};
	\node at (\mx+1*\xgap, \my-5*\ygap) {*};
	\node at (\mx+2*\xgap, \my-0*\ygap) {0};
	\node at (\mx+2*\xgap, \my-1*\ygap) {0};
	\node at (\mx+2*\xgap, \my-2*\ygap) {1};
	\node at (\mx+2*\xgap, \my-3*\ygap) {*};
	\node at (\mx+2*\xgap, \my-4*\ygap) {*};
	\node at (\mx+2*\xgap, \my-5*\ygap) {*};
	\node at (\mx+3*\xgap, \my-0*\ygap) {0};
	\node at (\mx+3*\xgap, \my-1*\ygap) {0};
	\node at (\mx+3*\xgap, \my-2*\ygap) {0};
	\node at (\mx+3*\xgap, \my-3*\ygap) {1};
	\node at (\mx+3*\xgap, \my-4*\ygap) {*};
	\node at (\mx+3*\xgap, \my-5*\ygap) {*};
	\node at (\mx+4*\xgap, \my-0*\ygap) {0};
	\node at (\mx+4*\xgap, \my-1*\ygap) {0};
	\node at (\mx+4*\xgap, \my-2*\ygap) {0};
	\node at (\mx+4*\xgap, \my-3*\ygap) {0};
	\node at (\mx+4*\xgap, \my-4*\ygap) {1};
	\node at (\mx+4*\xgap, \my-5*\ygap) {*};
	\node at (\mx+5*\xgap, \my-0*\ygap) {0};
	\node at (\mx+5*\xgap, \my-1*\ygap) {0};
	\node at (\mx+5*\xgap, \my-2*\ygap) {0};
	\node at (\mx+5*\xgap, \my-3*\ygap) {0};
	\node at (\mx+5*\xgap, \my-4*\ygap) {0};
	\node at (\mx+5*\xgap, \my-5*\ygap) {1};
	
	\node at (\mx+2.5*\xgap, \my-6.5*\ygap) {Eluder sequence};
	
	\def \mx{4.5}
	\drawgrid
	\node at (\mx+0*\xgap, \my-0*\ygap) {1};
	\node at (\mx+0*\xgap, \my-1*\ygap) {0};
	\node at (\mx+0*\xgap, \my-2*\ygap) {0};
	\node at (\mx+0*\xgap, \my-3*\ygap) {0};
	\node at (\mx+0*\xgap, \my-4*\ygap) {0};
	\node at (\mx+0*\xgap, \my-5*\ygap) {0};
	\node at (\mx+1*\xgap, \my-0*\ygap) {0};
	\node at (\mx+1*\xgap, \my-1*\ygap) {1};
	\node at (\mx+1*\xgap, \my-2*\ygap) {0};
	\node at (\mx+1*\xgap, \my-3*\ygap) {0};
	\node at (\mx+1*\xgap, \my-4*\ygap) {0};
	\node at (\mx+1*\xgap, \my-5*\ygap) {0};
	\node at (\mx+2*\xgap, \my-0*\ygap) {0};
	\node at (\mx+2*\xgap, \my-1*\ygap) {0};
	\node at (\mx+2*\xgap, \my-2*\ygap) {1};
	\node at (\mx+2*\xgap, \my-3*\ygap) {0};
	\node at (\mx+2*\xgap, \my-4*\ygap) {0};
	\node at (\mx+2*\xgap, \my-5*\ygap) {0};
	\node at (\mx+3*\xgap, \my-0*\ygap) {0};
	\node at (\mx+3*\xgap, \my-1*\ygap) {0};
	\node at (\mx+3*\xgap, \my-2*\ygap) {0};
	\node at (\mx+3*\xgap, \my-3*\ygap) {1};
	\node at (\mx+3*\xgap, \my-4*\ygap) {0};
	\node at (\mx+3*\xgap, \my-5*\ygap) {0};
	\node at (\mx+4*\xgap, \my-0*\ygap) {0};
	\node at (\mx+4*\xgap, \my-1*\ygap) {0};
	\node at (\mx+4*\xgap, \my-2*\ygap) {0};
	\node at (\mx+4*\xgap, \my-3*\ygap) {0};
	\node at (\mx+4*\xgap, \my-4*\ygap) {1};
	\node at (\mx+4*\xgap, \my-5*\ygap) {0};
	\node at (\mx+5*\xgap, \my-0*\ygap) {0};
	\node at (\mx+5*\xgap, \my-1*\ygap) {0};
	\node at (\mx+5*\xgap, \my-2*\ygap) {0};
	\node at (\mx+5*\xgap, \my-3*\ygap) {0};
	\node at (\mx+5*\xgap, \my-4*\ygap) {0};
	\node at (\mx+5*\xgap, \my-5*\ygap) {1};
	\node at (\mx+2.5*\xgap, \my-6.5*\ygap) {Star sequence};
	
	\def \mx{9}
	\drawgrid
	\node at (\mx+0*\xgap, \my-0*\ygap) {1};
	\node at (\mx+0*\xgap, \my-1*\ygap) {1};
	\node at (\mx+0*\xgap, \my-2*\ygap) {1};
	\node at (\mx+0*\xgap, \my-3*\ygap) {1};
	\node at (\mx+0*\xgap, \my-4*\ygap) {1};
	\node at (\mx+0*\xgap, \my-5*\ygap) {1};
	\node at (\mx+1*\xgap, \my-0*\ygap) {0};
	\node at (\mx+1*\xgap, \my-1*\ygap) {1};
	\node at (\mx+1*\xgap, \my-2*\ygap) {1};
	\node at (\mx+1*\xgap, \my-3*\ygap) {1};
	\node at (\mx+1*\xgap, \my-4*\ygap) {1};
	\node at (\mx+1*\xgap, \my-5*\ygap) {1};
	\node at (\mx+2*\xgap, \my-0*\ygap) {0};
	\node at (\mx+2*\xgap, \my-1*\ygap) {0};
	\node at (\mx+2*\xgap, \my-2*\ygap) {1};
	\node at (\mx+2*\xgap, \my-3*\ygap) {1};
	\node at (\mx+2*\xgap, \my-4*\ygap) {1};
	\node at (\mx+2*\xgap, \my-5*\ygap) {1};
	\node at (\mx+3*\xgap, \my-0*\ygap) {0};
	\node at (\mx+3*\xgap, \my-1*\ygap) {0};
	\node at (\mx+3*\xgap, \my-2*\ygap) {0};
	\node at (\mx+3*\xgap, \my-3*\ygap) {1};
	\node at (\mx+3*\xgap, \my-4*\ygap) {1};
	\node at (\mx+3*\xgap, \my-5*\ygap) {1};
	\node at (\mx+4*\xgap, \my-0*\ygap) {0};
	\node at (\mx+4*\xgap, \my-1*\ygap) {0};
	\node at (\mx+4*\xgap, \my-2*\ygap) {0};
	\node at (\mx+4*\xgap, \my-3*\ygap) {0};
	\node at (\mx+4*\xgap, \my-4*\ygap) {1};
	\node at (\mx+4*\xgap, \my-5*\ygap) {1};
	\node at (\mx+5*\xgap, \my-0*\ygap) {0};
	\node at (\mx+5*\xgap, \my-1*\ygap) {0};
	\node at (\mx+5*\xgap, \my-2*\ygap) {0};
	\node at (\mx+5*\xgap, \my-3*\ygap) {0};
	\node at (\mx+5*\xgap, \my-4*\ygap) {0};
	\node at (\mx+5*\xgap, \my-5*\ygap) {1};
	\node at (\mx+2.5*\xgap, \my-6.5*\ygap) {Threshold sequence};
\end{tikzpicture}

%
%
%
%